\documentclass[10pt,journal,compsoc]{IEEEtran}

\usepackage{xspace,amsmath,amssymb,amsthm,amsfonts,epsfig,syntonly,dsfont,pifont}
\usepackage{cite,bm,color,url,textcomp}
\usepackage{algorithmicx,algorithm,algpseudocode}
\usepackage{epstopdf}
\usepackage{empheq,float}
\usepackage{graphicx,subfigure,balance}
\usepackage{multirow}

\usepackage{makecell}
\usepackage{colortbl}

\usepackage[numbers]{natbib}
\usepackage[dvipsnames]{xcolor}
\usepackage[colorlinks]{hyperref}
\usepackage{enumitem}

\newtheorem{assumption}{Assumption}

\newtheorem{lemma}{Lemma}
\newtheorem{corollary}{Corollary}
\newtheorem{theorem}{Theorem}
\newtheorem{definition}{Definition}
\theoremstyle{definition}

\algrenewcommand{\algorithmiccomment}[1]{\hspace{0.5cm}// #1}

\definecolor{mygray}{gray}{0.7}
\newcommand\gray[1]{{\color{mygray}#1}}

\DeclareMathOperator*{\argmax}{arg\,max}

\graphicspath{{./figs/}}

\usepackage{longtable, array, booktabs}

\usepackage{times}

\begin{document}

\title{VASSO: Variance Suppression for Sharpness-Aware Minimization}
\author{Bingcong Li, Yilang Zhang, and Georgios B. Giannakis \\
	
\thanks{
B. Li is with Department of Computer Science at ETH Z{\"u}rich, 8092 Z{\"u}rich, Switzerland. Email: \texttt{bingcong.li@inf.ethz.ch}.
Y. Zhang is with ML Research, Morgan Stanley, New York, NY 10019, USA. Email: \texttt{yilang.zhang@morganstanley.com.}
G. B. Giannakis is with Department of Electrical and Computer Engineering, Univ. of Minnesota, Minneapolis, MN 55455, USA. Email:\texttt{georgios@umn.edu.}

Most of this work was performed when B. Li and Y. Zhang were at the Univ. of Minnesota. 
	}
}

\markboth{IEEE TRANSACTIONS on Pattern Analysis and Machine Intelligence (submitted)}{}

\IEEEtitleabstractindextext{
\begin{abstract}
Sharpness-aware minimization (SAM) has well-documented merits in enhancing generalization of deep neural network models. Accounting for sharpness in the loss function geometry, where neighborhoods of `flat minima' heighten generalization ability, SAM seeks `flat valleys' by minimizing the maximum loss provoked by an \textit{adversarial} perturbation within the neighborhood. Although critical to account for sharpness of the loss function, in practice SAM suffers from `\textit{over-friendly} adversaries,' which can curtail the outmost level of generalization. To avoid such `friendliness,' the present contribution fosters stabilization of adversaries through \textit{variance suppression} (VASSO). VASSO offers a general approach to \textit{provably} stabilize adversaries. In particular, when integrating VASSO with SAM, improved generalizability is numerically validated on extensive vision and language tasks. Once applied on top of a computationally efficient SAM variant, VASSO offers a desirable generalization-computation tradeoff. 
%
\end{abstract}
}

\maketitle

\section{Introduction}
Deep neural network (DNN) models have advanced the notion of ``learning from data,'' and they have markedly improved performance across various application tasks in vision and language \citep{devlin2018bert,gpt3}. Unfortunately, their overparametrization renders them prone to overfit on training data \citep{zhang2016}, which hampers their generalization ability on unseen data.  This shortcoming has been underscored in practice, and points to a gap in evaluating performance of training.

Common approaches to improving generalizability of DNNs include regularization and data augmentation \citep{dropout2014}. While it is a default choice to integrate regularization such as weight decay and dropout when training in practice, these methods are often insufficient for DNNs, especially when dealing with complicated network architectures \citep{sam4vit}. Another line of efforts resorts to suitable optimization schemes, attempting to find a generalizable local minimum. For example, stochastic gradient descent (SGD) outperforms Adam on certain overparameterized problems owing to its convergence to maximum margin solutions \citep{wilson2017}. Decoupling weight decay from Adam has been empirically seen to facilitate generalizability for many language tasks \citep{loshchilov2017adamw}. Unfortunately, the underlying mechanism promoting generalization remains elusive, and whether the generalization merits carry over to other intricate learning problems calls for extra theoretical investigations.

Our main focus is sharpness-aware minimization (SAM) -- a compelling optimization approach that facilitates state-of-the-art generalizability by exploiting areas of sharpness and flatness in the loss landscape \citep{foret2021,sam4vit}. A high-level interpretation of sharpness is how markedly the loss fluctuates in the neighboring parameter space. Large-scale empirical studies have shown that sharpness-based measures highly correlate with generalization \citep{jiang2020}, and flat minima improve generalizability \citep{keskar2016,foret2021,sam4vit}. This can be conceptually understood using Fig. \ref{fig.example}, where the test loss only slightly increases in flat valleys under distributional shifts. A number of approaches  have embraced sharpness to boost generalization. The work in \cite{keskar2016} suggests that the batch size of SGD impresses solution flatness. Entropy SGD leverages the local entropy in search of a flat valley \citep{pratik2017}. Different from prior works, SAM induces flatness by explicitly minimizing the \textit{adversarially} perturbed loss, defined as the maximum loss of a neighboring area. Through such a perturbed loss, SAM has boosted generalization in various vision and language tasks \citep{sam4vit,gasam2022}. The mechanism behind SAM's success has been theoretically investigated based on arguments of implicit regularization; see e.g., \citep{maksym2022,wen2023,bartlett2022dynamics}. 

The notion of a perturbing adversary, or the \textit{adversary} for short, is central to SAM's improved generalization because it effectively measures sharpness through the loss difference with the original model \citep{foret2021,zhuang2022, kim2022}. In practice however, accounting for sharpness is undermined by what can be viewed as a \textit{friendly adversary}. Confined by the stochastic linearization for computational efficiency, SAM's adversary only captures the sharpness for a particular minibatch, and can become a friend on other data samples. Because `global sharpness' is not approached accurately, the friendly adversary challenges SAM from attaining its utmost generalizability. To overcome this challenge, the present work advocates \underline{va}riance \underline{s}uppre\underline{s}si\underline{o}n (VASSO\footnote{VASSO coincides with the Greek nickname for Vasiliki.}) to alleviate `friendliness' by stabilizing adversaries. VASSO is a general approach that can be seamlessly integrated with SAM variants to theoretically and numerically demonstrate gain in generalization. This work focuses particularly on two of the most valuable use cases.

The first one is to integrate VASSO with vanilla SAM. The resultant algorithm, also referred to as VASSO, has a \textit{provable} stabilized adversary that showcases favorable numerical performance over SAM for a wide spectrum of deep learning tasks. 

The second application of VASSO manifests in trading off generalization for computational efficiency. At one extreme, the drastically improved generalization of SAM incurs the cost of two backpropagations per iteration. The opposite extreme is SGD, where low computational cost comes at the price of a moderate level of generalization. Recent research has enriched this tradeoff by developing lightweight SAM variants. For instance, LookSAM computes the extra stochastic gradient once every few iterations, and reuses it in a fine-grained manner to approximate the additional gradient \citep{liu2022}. ESAM obtains its adversarial vector based on stochastic weight perturbation, and further saves computation by selecting a subset of the minibatch data for gradient computation \citep{du2022}. The computational burden of SAM can also be lowered by switching between SAM and SGD following a predesigned schedule \citep{zhao222}, or adaptively as in \citep{jiang2023}. SAF connects SAM with distillation to reduce computational complexity~\citep{du2022saf}. F-SAM decomposes adversarial perturbations into full-gradient and stochastic-noise components, and leverages the later as a ``friendly'' perturbation that sharpens the batch loss yet minimizing its effect on other samples~\citep{F-SAM}. However, most of these works reuse the stochastic linearization of SAM, which reduces their ability to cope with friendly adversaries, and prevents them from attaining a desirable generalization-computation tradeoff. To this end, VASSO is combined here with a computationally efficient SAM variant \citep{zhao222}. For a prescribed computational budget, the resultant algorithm termed  efficient VASSO (eVASSO), leads to markedly improved generalization. 

\begin{figure}[t]
	\centering
	\includegraphics[width=.45\textwidth]{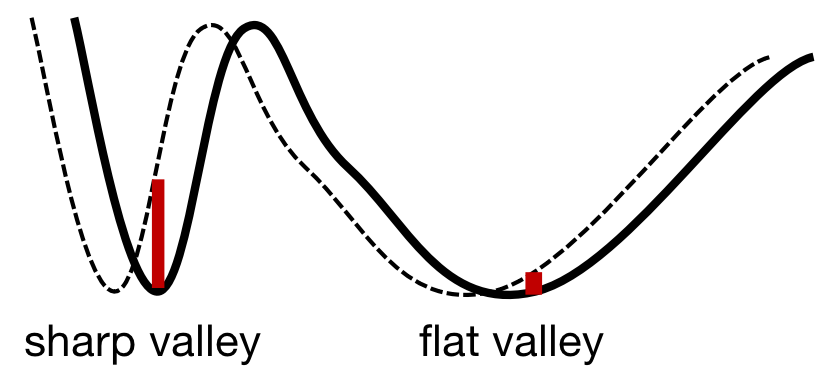}
	\caption{Intuition on why flat minima boost generalization. Solid (dotted) curve denotes training (test) loss. The red bar indicates the gap between training and test loss. Clearly, this gap is smaller on flat valleys.}
	\label{fig.example}
\end{figure}

All in all, our contributions can be summarized as follows.

\noindent
\begin{enumerate}
	\item[\ding{118}] A \textit{friendly adversary} is identified as an obstacle challenging generalizability of SAM. Experiments demonstrate that it can even nullify the generalization benefits.
	
	\item[\ding{118}] \textit{Variance suppression} is developed to handle this issue by stabilizing adversaries. The theoretically guaranteed stability promotes refined global sharpness estimates, and thus alleviates the impact of friendly adversaries. The merits of VASSO are demonstrated also experimentally on tasks such as image classification, domain generalization, label noise, and neural machine translation.
	
\item[\ding{118}]
VASSO improves the generalization-computation tradeoff when integrated with computationally efficient SAM variants such as eSAM. For the same generalization level, the resultant approach, eVASSO, saves $57\%$ and $70\%$ computational overhead relative to eSAM and SAM, respectively. 
	
\end{enumerate}

This work broadens and innovates over \citep{li2023vasso} in three directions: (i) a computationally efficient algorithm is developed and analyzed to desirably trade off generalization for computation in sharpness-aware minimization; (ii) numerical tests incorporate more challenging tasks such as domain generalization; and (iii) further elaboration and insights are included to improve intuition. 

\textbf{Notation}. Bold lowercase (capital) letters denote column vectors (matrices); $\| \mathbf{x}\|$ stands for $\ell_2$ norm of vector $\mathbf{x}$; and $\langle \mathbf{x}, \mathbf{y} \rangle$ is the inner product of $\mathbf{x}$ and $\mathbf{y}$. $\mathbb{S}_\rho(\mathbf{x})$ denotes the surface of a ball with radius $\rho$ centered at $\mathbf{x}$, i.e., $\mathbb{S}_\rho(\mathbf{x}):= \{ \mathbf{x} + \rho \mathbf{u} ~|~ \| \mathbf{u}\| = 1 \}$.

\section{The known, the good, and the challenge}

This section starts with a recap of SAM (i.e., the known), followed by refined analyses and SAM's 
convergence upshot (i.e., the good). Next, the notion of a \textit{friendly adversary} that confines generalizability is elaborated and illustrated numerically.

\subsection{The known}

Aiming at a minimum with a flat basin, SAM enforces small loss around the entire neighborhood in the parameter space \citep{foret2021}. This idea is formulated as a minimax problem
\begin{align}\label{eq.prob}
	\min_{\mathbf{x}} \max_{\| \bm{\epsilon} \| \leq \rho} f (\mathbf{x} + \bm{\epsilon} )
\end{align}
with $\rho$ denoting the radius of the considered neighborhood, and $f(\mathbf{x}):= \mathbb{E}_{\cal B}[f_{\cal B}(\mathbf{x})]$ the possibly nonconvex objective, where $\mathbf{x}$ is the neural network parameter, and ${\cal B}$ is a random minibatch of data. Formulation \eqref{eq.prob} captures the  implicit sharpness measure $\max_{\| \bm{\epsilon} \| \leq \rho} f \big(\mathbf{x} + \bm{\epsilon} \big) - f(\mathbf{x})$, which effectively drives the optimization trajectory towards the desirable flat valley \citep{kim2022}.

The inner maximization in \eqref{eq.prob} has a natural interpretation as finding an \textit{adversarial model} for $\mathbf{x}_t$, where $t$ denotes the iteration index. Critical as it is, such an adversary prompts \textit{stochastic linearization} to avoid full gradient computation, that is,
\begin{align}\label{eq.sam_epsilon_full}
	\bm{\epsilon}_t & = \argmax_{\| \bm{\epsilon} \| \leq \rho} f(\mathbf{x}_t + \bm{\epsilon}) \stackrel{(a)}{\approx} \argmax_{\| \bm{\epsilon} \| \leq \rho} f(\mathbf{x}_t) + \langle \nabla f( \mathbf{x}_t), \bm{\epsilon} \rangle \nonumber \\
	& \stackrel{(b)}{\approx} \argmax_{\| \bm{\epsilon} \| \leq \rho} f(\mathbf{x}_t) + \langle \mathbf{g}_t(\mathbf{x}_t), \bm{\epsilon} \rangle
\end{align}
where linearization $(a)$ relies on the first-order Taylor expansion of $f(\mathbf{x}_t + \bm{\epsilon})$ that is typically accurate when choosing $\rho$ small. Consider next replacing 
$\nabla f(\mathbf{x}_t)$ in $(b)$ with a stochastic gradient $\mathbf{g}_t(\mathbf{x}_t)$ on minibatch ${\cal B}_t$  to reduce the computational burden of $\nabla f(\mathbf{x}_t)$. Catalyzed by the stochastic linearization in \eqref{eq.sam_epsilon_full}, it is possible to express SAM's adversary in closed form as
\begin{empheq}[box=\fbox]{align}\label{eq.sam_epsilon}
	\textbf{SAM:}~~~~\bm{\epsilon}_t = \rho \frac{\mathbf{g}_t(\mathbf{x}_t)}{ \| \mathbf{g}_t(\mathbf{x}_t) \| }.
\end{empheq}
SAM then adopts the adversarial stochastic gradient $\mathbf{g}_t(\mathbf{x}_t+\bm{\epsilon}_t)$ to update $\mathbf{x}_t$ per SGD fashion. A step-by-step implementation is listed in Alg. \ref{alg.sam}, where the means to find an adversary in lines 5 and 6 is presented in a generic form in order to unify the algorithmic framework with approaches in subsequent sections. 

\begin{algorithm}[t]
    \caption{Generic form of SAM} \label{alg.sam}
    \begin{algorithmic}[1]
    	\State \textbf{Initialize:} $\mathbf{x}_0, \rho$
    	\For {$t=0,\dots,T-1$}
    		\State Sample a minibatch ${\cal B}_t$
    		\State Calculate stochastic gradient $\mathbf{g}_t(\mathbf{x}_t)$ on ${\cal B}_t$
    		\State Find $\bm{\epsilon}_t \in \mathbb{S}_\rho(\mathbf{0})$ via stochastic linearization; e.g.,  \\
    		\hspace{0.8cm}  \Comment{\eqref{eq.vso} for \textbf{VASSO}}, or \eqref{eq.sam_epsilon} for \textbf{SAM}
    		\State Calculate stochastic gradient $\mathbf{g}_t(\mathbf{x}_t + \bm{\epsilon}_t)$ on ${\cal B}_t$
			\State Update model via $\mathbf{x}_{t+1} = \mathbf{x}_t - \eta \mathbf{g}_t(\mathbf{x}_t + \bm{\epsilon}_t)$
		\EndFor
		\State \textbf{Return:} $\mathbf{x}_T$
	\end{algorithmic}
\end{algorithm}

\subsection{The good}

For an insightful understanding of SAM, this subsection focuses on Alg. \ref{alg.sam}, and establishes convergence for the solver of \eqref{eq.prob}. Assumptions to this end that are common for nonconvex stochastic optimization are listed next \citep{ghadimi2013,bottou2018,mi2022,zhuang2022}.

\begin{assumption}[lower bounded loss]\label{as.1}
	 Function $f(\mathbf{x})$ is bounded from below, that is, $ \exists~ f^* > -\infty$ such that $f(\mathbf{x}) \geq f^*,\forall \mathbf{x}$.
\end{assumption}
\begin{assumption}[smoothness]\label{as.2}
	The stochastic gradient $\mathbf{g}(\mathbf{x})$ is $L$-Lipschitz, i.e., $\| \mathbf{g}(\mathbf{x}) - \mathbf{g}(\mathbf{y}) \| \leq L \| \mathbf{x}  - \mathbf{y}\|, \forall \mathbf{x}, \mathbf{y}$.
\end{assumption}
\begin{assumption}[bounded variance]\label{as.3}
	 The stochastic gradient $\mathbf{g}(\mathbf{x})$ is unbiased with bounded variance, that is, $\mathbb{E} [\mathbf{g}(\mathbf{x}) | \mathbf{x}] = \nabla f(\mathbf{x})$ and $\mathbb{E} [\| \mathbf{g}(\mathbf{x}) - \nabla f(\mathbf{x}) \|^2 | \mathbf{x}] \leq \sigma^2$ for some finite $\sigma$.
\end{assumption}

Since $\| \bm{\epsilon}_t \| = \rho$ holds for every $t$, the pertinent constraint in \eqref{eq.prob} is never violated; see lines 5 and 6 in Alg. \ref{alg.sam}. Thus, SAM convergence relates to the behavior of the objective for which a tight result is derived next.

\begin{theorem}[SAM convergence]\label{thm.sam}
	If Assumptions \ref{as.1} -- \ref{as.3} hold, $\eta_t \equiv \eta = \frac{\eta_0}{ \sqrt{T}} \leq \frac{2}{3L}$, and $\rho = \frac{\rho_0}{\sqrt{T}}$, then with $c_0 = 1 - \frac{3L\eta}{2} \in (0,1)$ 
	 Alg. \ref{alg.sam} guarantees that 
	\begin{align*}
		&\frac{1}{T}\sum_{t=0}^{T-1}\mathbb{E}\big[ \| \nabla f(\mathbf{x}_t )\|^2 \big] = {\cal O} \bigg( \frac{\sigma^2}{\sqrt{T}} \bigg) ~~~~\text{and}~~~~  \\
		& \frac{1}{T}\sum_{t=0}^{T-1}\mathbb{E}\big[ \| \nabla f(\mathbf{x}_t + \bm{\epsilon}_t)\|^2 \big] = {\cal O} \bigg( \frac{\sigma^2}{\sqrt{T}} \bigg).
	\end{align*}
\end{theorem}
Within a constant factor SAM's convergence rate is the same as that of SGD; see Appendix \ref{apdx.sec.proof} for such a factor included in big $\cal O$. Our result avoids the need for a bounded gradient assumption present in prior analyses~\citep{mi2022,zhuang2022}. 
Notably, the scaling $\rho=\mathcal{O}(1/\sqrt{T})$ ensures that the perturbation-dependent terms vanish asymptotically. When $\rho$ is held constant, existing analyses establish convergence up to a neighborhood of stationarity whose size depends on $\rho$~\cite{NEURIPS2023_5305b789}. Alternative choices, such as $\rho=\mathcal{O}(1/T^{1/4})$, are possible under stronger assumption of individual smoothness of per-sample loss~\cite{maksym2022}. The present analysis instead relies only on global smoothness of $f$; cf. Assumption~\ref{as.2}.

A message from Theorem \ref{thm.sam} is that \textit{any} adversary satisfying $\bm{\epsilon}_t \in \mathbb{S}_\rho(\mathbf{0})$ ensures converge. Because the surface $\mathbb{S}_\rho(\mathbf{0})$ is already a huge space, it challenges the plausible optimality of the adversary, and raises a natural question: \textit{Is it possible to find a more powerful adversary to enhance the generalization ability?} 

\subsection{The challenge: friendly adversary}

\textbf{Adversary to a minibatch can be a friend of other minibatches.} 
SAM's adversary can be `malicious' for a minibatch ${\cal B}_t$ used in iteration $t$, but not necessarily for other data, because it only safeguards $f_{{\cal B}_t} (\mathbf{x}_t + \bm{\epsilon}_t) - f_{{\cal B}_t} (\mathbf{x}_t ) \geq 0$ for a small $\rho$. Given another minibatch data ${\cal B}$, it can be shown that $f_{{\cal B}} (\mathbf{x}_t + \bm{\epsilon}_t) - f_{{\cal B}} (\mathbf{x}_t ) \leq 0$, whenever the stochastic gradients do not align well, meaning $\langle \mathbf{g}_t(\mathbf{x}_t),  \mathbf{g}_{\cal B}(\mathbf{x}_t) \rangle \leq 0$. Note that such misalignment is common because the variance is massive in large-scale training datasets. This issue will be referred to as \textit{friendly adversary}, and it implies that the adversary vector $\bm{\epsilon}_t$ cannot accurately depict the global sharpness of $\mathbf{x}_t$. The `friendly adversary' also has a more involved interpretation, that is, $\mathbf{g}_t(\mathbf{x}_t)$  falls outside the column space of Hessian at convergence; see \citep[Definition 4.3]{wen2023} for further elaboration. This misalignment of high-order derivatives undermines the inductive bias of SAM, thereby worsening generalization.

\begin{figure*}[t]
	\centering
	\begin{tabular}{ccc}
		\hspace{-0.4cm}
		\includegraphics[width=.32\textwidth]{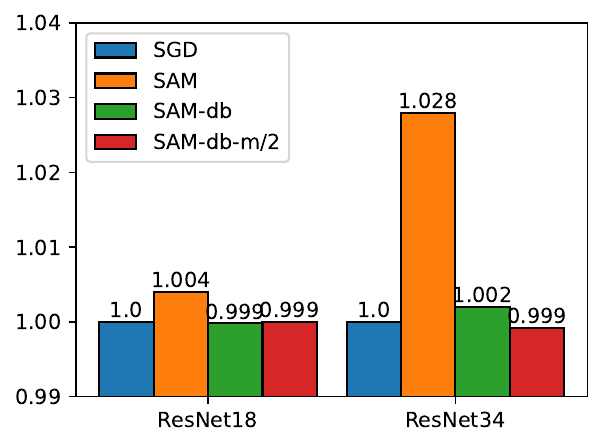}&
		\hspace{-0.3cm}
		\includegraphics[width=.32\textwidth]{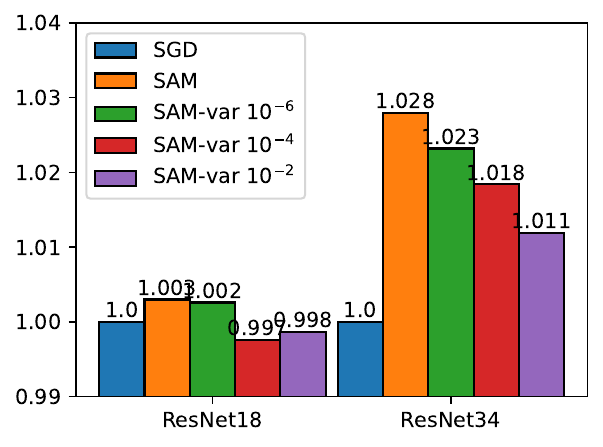}&
		\hspace{-0.3cm}
		\includegraphics[width=.32\textwidth]{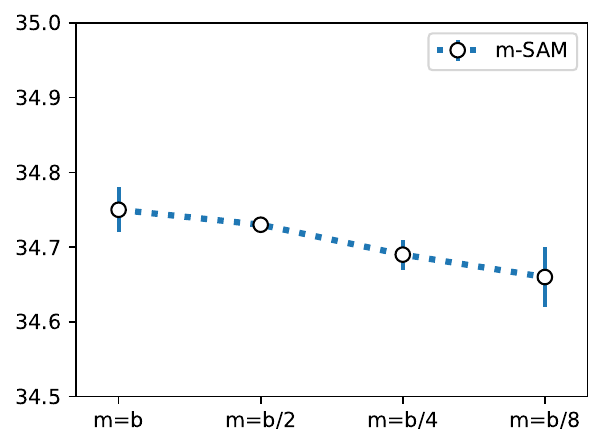}
		\\ 
	  (a)  & ~~~~~(b) & ~~~~~(c)
	\end{tabular}
	\vspace{-0.2cm}
	\caption{(a) A friendly adversary diminishes the generalization ability of SAM; (b) $m$-sharpness may \textit{not} directly correlate with variance since noisy gradient degrades generalization; and (c) $m$-sharpness may not hold universally. Note that test accuracies in (a) and (b) are normalized to SGD.}
	 \label{fig.m-sharpness}
\end{figure*}

\begin{figure*}[t]
	\centering
	\begin{tabular}{ccccc}
		\hspace{-0.2cm}
		\includegraphics[width=.191\textwidth]{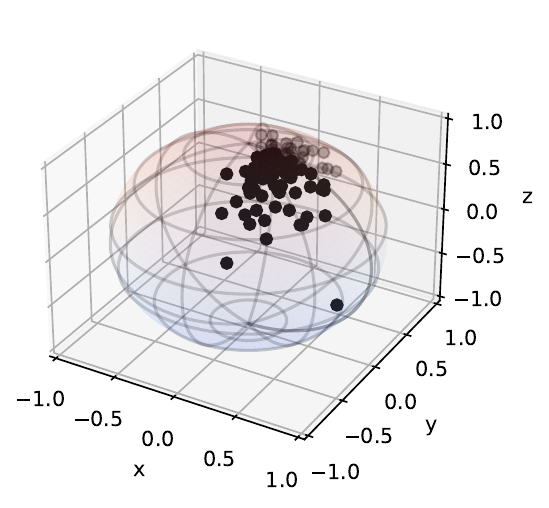}&
		\hspace{-0.5cm}
		\includegraphics[width=.191\textwidth]{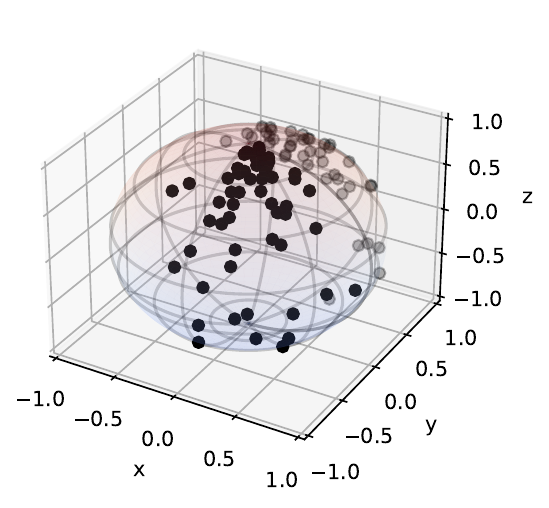}&
		\hspace{-0.5cm}
		\includegraphics[width=.191\textwidth]{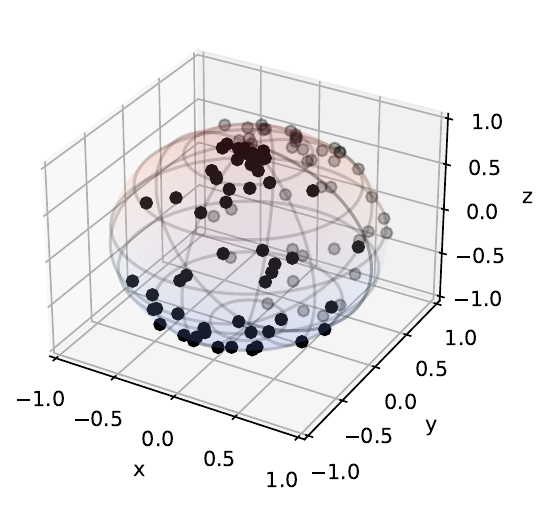}&
		\hspace{-0.5cm}
		\includegraphics[width=.191\textwidth]{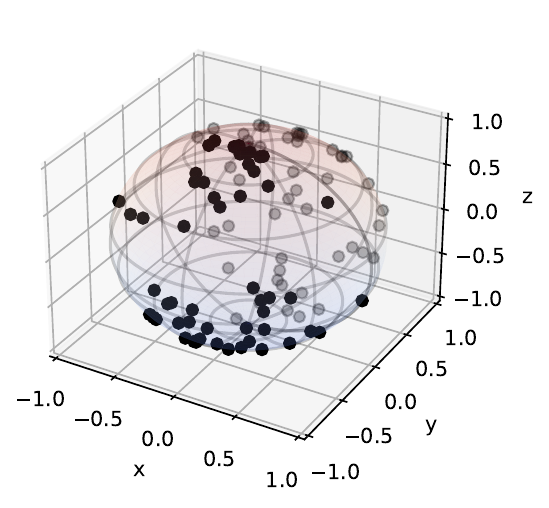}&
		\hspace{-0.3cm}
		\includegraphics[width=.17\textwidth]{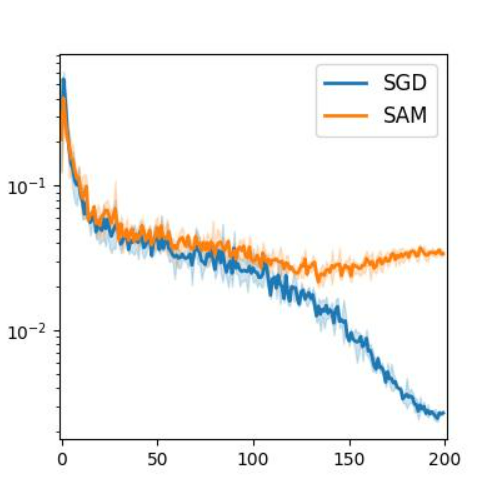}
		\\ 
		\hspace{-0.15cm}
	  (a) SNR $=5$ & \!\!(b) SNR $=1$ & \! (c) SNR $=0.1$ & \! (d) SNR $=0.01$ & (e) SNR in practice 
	\end{tabular}
	\caption{(a) - (d) SAM's adversaries are spread across the sphere; (e) SNR is in $[0.01, 0.1]$ when training a ResNet-18 on CIFAR10, where the SNR is calculated at the first iteration of every epoch.}
	 \label{fig.noise}
	 \vspace{-0.3cm}
\end{figure*}

To numerically visualize the catastrophic impact of the friendly adversary, we manually introduce one by replacing line 5 of Alg. \ref{alg.sam} as $\tilde{\bm{\epsilon}}_t = \rho \tilde{\mathbf{g}}_t(\mathbf{x}_t) / \| \tilde{\mathbf{g}}_t(\mathbf{x}_t) \| $, where $\tilde{\mathbf{g}}_t$ denotes the gradient on $\tilde{\cal B}_t$, a randomly sampled batch of the same size as ${\cal B}_t$. This modified approach is denoted as SAM-db, and its performance for i) ResNet-18 on CIFAR10; and, ii) ResNet-34 on CIFAR100\footnote{\url{https://www.cs.toronto.edu/~kriz/cifar.html}} can be found in Fig. \ref{fig.m-sharpness}(a). Note that the test accuracy is normalized relative to SGD for the ease of visualization. It is evident that the friendly adversary $\tilde{\bm{\epsilon}}_t$ in SAM-db almost nullifies the generalization benefits entirely.

\textbf{Source of a friendly adversary.}
The major cause behind a friendly adversary is due to the gradient variance, which equivalently translates to the lack of stability in SAM's stochastic linearization $(2b)$. An illustrative 
three-dimensional example is shown in Fig. \ref{fig.noise}, where we plot the adversary $\bm{\epsilon}_t$ obtained from different minibatch data. The minibatch gradient is simulated by adding Gaussian noise to the true gradient. When the signal to noise ratio (SNR) is similar to a practical scenario (ResNet-18 on CIFAR10 in Fig. \ref{fig.noise} (e)), it can be seen in Fig. \ref{fig.noise} (c) and (d) that the adversaries \textit{almost uniformly} spread over the sphere, which strongly indicates how challenging is to evaluate sharpness. 

\textbf{Friendly adversary under the lens of Frank Wolfe.} An additional evidence in support of SAM's friendly adversary resides in its link with stochastic Frank Wolfe (SFW), which also relies heavily on stochastic linearization \citep{reddi2016}. The stability of SFW is known to be vulnerable, and its convergence cannot be guaranteed without a sufficiently large batch size. Appendix \ref{apdx.sec.1} provides the means to obtain friendly adversaries in SAM, and demonstrates that this is tantamount to one-step SFW with a \textit{constant} batch size. This unveils possible instability of SAM's stochastic linearization.

\subsection{A closer look at friendly adversaries}

Gradient variance is a major driver behind SAM's friendly adversary and unstable stochastic linearization. At first glance however, this seems to conflict with the \textit{empirical} 
notion of $m$-sharpness, which asserts that the benefit of SAM is more pronounced when $\bm{\epsilon}_t$ is found using subsampled ${\cal B}_t$ of size $m$, meaning larger variance. Concretely, vanilla SAM aims to improve the global sharpness in~\eqref{eq.prob} through the population mean
$$
\min_{\mathbf{x}} \max_{\| \bm{\epsilon} \| \leq \rho} \mathbb{E}_{\mathcal{B}} \big[ f_{\mathcal{B}} (\mathbf{x} + \bm{\epsilon} ) \big]. 
$$
In contrast, $m$-SAM considers the altered objective
$$
\min_{\mathbf{x}} \mathbb{E}_{\mathcal{B}} \Big[\max_{\| \bm{\epsilon} \| \leq \rho}  f_{\mathcal{B}} (\mathbf{x} + \bm{\epsilon} ) \Big]
$$
optimizing batch-specific sharpness. 
While F-SAM~\cite{F-SAM} also investigates friendly adversaries and recruits EMA as a stable estimation of global gradient $\nabla f(\mathbf{x})$, its analysis is built upon $m$-sharpness, for which a more friendly perturbation is preferred to minimize its adverse effects on other batches. Consequently, F-SAM projects $\mathbf{g}_t(\mathbf{x}_t)$ onto the orthogonal direction of $\mathbf{d}_t$ to eliminate the full gradient component.

Since $m$-sharpness hinges heavily on the loss curvature, it is unlikely to hold universally. For example, a transformer is trained on the IWSLT-14 dataset~\citep{vaswani2017attention}, where the test performance (BLEU) decreases with smaller $m$ even if $\rho$ has been tuned carefully; see Fig. \ref{fig.m-sharpness}(c). In principle, $m$-sharpness is not necessarily related to sharpness or generalization, as verified by the example given in \citep[Sec. 3]{maksym2022}. Moreover, $m$-sharpness formulation can be ill-posed for a specific choice of $m$; see Appendix \ref{apdx.sec.m-sharpness} for further details.

Even in the regime where $m$-sharpness is empirically observed (as with ResNet-18 on CIFAR10 and ResNet-34 on CIFAR100), we have confirmed experimentally that $m$-sharpness is \textit{not} a consequence of gradient variance; thus, there is contradiction with the friendly adversary pursued in this work. 

\textbf{Observation 1. Same variance, different generalization.} Let $m=128$ and batch size $b=128$. Recall the SAM-db experiment in Fig. \ref{fig.m-sharpness}(a). If $m$-sharpness were a consequence of gradient variance, it would be reasonable to expect that SAM-db has comparable performance to SAM simply because their batchzises (hence variance) for finding adversaries are the same. Unfortunately, SAM-db exhibits degraded accuracy. We further increase the variance of $\tilde{\mathbf{g}}_t(\mathbf{x}_t)$ by setting $m = 64$. The resultant algorithm is denoted as SAM-db-m/2. It does not catch up with SAM and performs even worse than SAM-db. These experiments validate that variance/stability correlates with the friendly adversary rather than $m$-sharpness.

\textbf{Observation 2. Enlarged variance degrades generalization.} We explicitly increase variance when finding the adversary by adding Gaussian noise $\bm{\zeta}$ to $\mathbf{g}_t(\mathbf{x}_t)$, i.e., $\hat{\bm{\epsilon}}_t = \rho \frac{ \mathbf{g}_t(\mathbf{x}_t) + \bm{\zeta}}{\|\mathbf{g}_t(\mathbf{x}_t) + \bm{\zeta}\|} $. After tuning for the best $\rho$ to compensate the variance of $\bm{\zeta}$, the test performance is plotted in Fig. \ref{fig.m-sharpness}(b). It can be seen that generalization merits clearly decrease with larger variance on both ResNet-18 and ResNet-34. This again illustrates that the plausible benefit of $m$-sharpness does not stem from increased variance.

While understanding $m$-sharpness is beyond the scope of this work, Observations 1 and 2 jointly suggest that gradient variance correlates with friendly adversaries rather than $m$-sharpness.

\subsection{The metric: characterization of a friendly adversary}

To characterize friendly adversaries analytically, it is convenient to start with necessary notation. Let the \textit{quality} of a stochastic linearization at $\mathbf{x}_t$ with slope $\mathbf{v}$ be ${\cal L}_t ( \mathbf{v} ):= \max_{\| \bm{\epsilon} \| \leq \rho} f(\mathbf{x}_t) + \langle \mathbf{v} ,  \bm{\epsilon} \rangle $. For example, ${\cal L}_t \big( \mathbf{g}_t(\mathbf{x}_t) \big)$ is the quality of SAM. Another critical case is ${\cal L}_t \big( \nabla f(\mathbf{x}_t) \big)$. 
It is shown in \citep{zhuang2022} and (2$b$) that ${\cal L}_t \big( \nabla f(\mathbf{x}_t) \big) \approx \max_{\| \bm{\epsilon} \|\leq \rho} f(\mathbf{x}_t + \bm{\epsilon})$ when $\rho$ is small. Moreover, ${\cal L}_t \big( \nabla f(\mathbf{x}_t) \big) - f(\mathbf{x}_t)$ is also an accurate approximation of sharpness. These observations safeguard ${\cal L}_t ( \nabla f(\mathbf{x}_t) )$ as the anchor when analyzing the stability of stochastic linearization.

\begin{definition}[$\delta$-stability]\label{def.stab}
	A stochastic linearization with slope $\mathbf{v}$ is said to be $\delta$-stable if its quality satisfies $\mathbb{E}\big[ | {\cal L}_t ( \mathbf{v})  - {\cal L}_t ( \nabla f(\mathbf{x}_t) ) | \big] \leq \delta$.
\end{definition}

Letting ${\cal L}_t \big( \nabla f(\mathbf{x}_t) \big) - f(\mathbf{x}_t)$ denote `sharpness,' Definition \ref{def.stab} can be decomposed as
\begin{align}
	{\cal L}_t ( \mathbf{v}) & - {\cal L}_t ( \nabla f(\mathbf{x}_t) ) \\
	& = \underbrace{\Big(	{\cal L}_t ( \mathbf{v}) - f(\mathbf{x}_t) \Big)}_{\text{estimated sharpness}} - \underbrace{\Big( {\cal L}_t ( \nabla f(\mathbf{x}_t) ) - f(\mathbf{x}_t) \Big)}_{\text{sharpness}}. \nonumber
\end{align}
The last equation suggests that $\delta$-stability reflects how well sharpness is estimated using slope $\mathbf{v}$. Hence, a larger $\delta$ implies a more friendly adversary, where the sharpness is approximated less accurately. Next, we will develop our novel approach to ensure improved $\delta$-stability over SAM. 

\section{Variance Suppression}
This section advocates variance suppression (VASSO) for SAM as a means of dealing with the friendly adversary. We start with the intuition behind VASSO, and then establish its improved $\delta$-stability over SAM. We also develop implementation aspects and possible extensions.

\subsection{Design and stability analysis}
A straightforward attempt towards stability is to equip SAM's stochastic linearization with variance-reduced gradients such as SVRG and SARAH \citep{johnson2013,nguyen2017,li2019bb, li2019l2s}. However, the requirement to compute a full gradient every few iterations is infeasible, and hardly scales for tasks such as training DNNs. 

VASSO overcomes this computational burden through a refined stochastic linearization. For a prescribed $\theta \in (0,1)$, VASSO is as follows. 
\begin{subequations}\label{eq.vso}
\begin{empheq}[box=\fbox]{align}
	\textbf{VASSO:}~~~~&\mathbf{d}_t = (1 - \theta)	\mathbf{d}_{t-1} + \theta \mathbf{g}_t(\mathbf{x}_t) \\
	&\bm{\epsilon}_t =  \argmax_{\| \bm{\epsilon} \| \leq \rho} f(\mathbf{x}_t) + \langle \mathbf{d}_t, \bm{\epsilon} \rangle = \rho \frac{\mathbf{d}_t}{ \| \mathbf{d}_t \| }.\label{eq.vso_b}
\end{empheq}
\end{subequations}

Compared with \eqref{eq.sam_epsilon_full} of SAM, the key difference is that VASSO relies on the slope $\mathbf{d}_t$ in \eqref{eq.vso_b} to promote a more stable stochastic linearization. Slope $\mathbf{d}_t$ is an exponentially moving average (EMA) of $\{ \mathbf{g}_t(\mathbf{x}_t) \}_t$ that smooths changes across consecutive iterations. Noticing that $\bm{\epsilon}_t$ and $\mathbf{d}_t$ share the same direction, the relatively smoothed vectors $\{\mathbf{d}_t\}_t$ thus ensure stability of $\{\bm{\epsilon}_t\}_t$ in VASSO. Moreover, as $\mathbf{d}_t$ processes information of different batches of data, the global sharpness can be captured in a principled manner to alleviate the friendly adversary challenge. 

It is worth noting that $\mathbf{d}_t$ is only used in the ascent step to construct the adversarial perturbation. The descent step remains intact with standard stochastic gradient; cf. line 8 of Algorithm~\ref{alg.sam}. Therefore, VASSO is designed to improve the reliability of the sharpness estimate used by SAM, rather than to reduce generic stochastic noise in SGD.

VASSO can be readily integrated with SAM, as summarized in Alg. \ref{alg.sam}. For convenience, the resultant algorithm is also referred to as VASSO. To theoretically characterize the effectiveness of VASSO, our first result considers $\mathbf{d}_t$ as a qualified strategy to estimate $\nabla f(\mathbf{x}_t)$, and delves into its mean-square error (MSE).

\begin{theorem}[Variance suppression]\label{thm.vso}
 Under Assumptions \ref{as.1} -- \ref{as.3}, if 
 i) $\bm{\epsilon}_t$ is obtained by \eqref{eq.vso} with $\theta \in (0, 1)$; and, ii) $\eta_t$ and $\rho$ are selected as in Theorem \ref{thm.sam}, then VASSO guarantees that the MSE of $\mathbf{d}_t$ is bounded by 
	\begin{align}\label{eq.VASSO_mse}
		\mathbb{E} \big[ \| \mathbf{d}_t - \nabla f(\mathbf{x}_t) \|	^2 \big] \leq \theta \sigma^2 + {\cal O}\bigg(\frac{(1-\theta)^2 \sigma^2}{\theta^2 \sqrt{T}}\bigg).
	\end{align}
\end{theorem}

Because SAM's gradient estimate has a looser bound on MSE (or variance), that is, $\mathbb{E}[  \| \mathbf{g}_t(\mathbf{x}_t) - \nabla f(\mathbf{x}_t) \|	^2  ] \leq \sigma^2$, the shrunk MSE in Theorem \ref{thm.vso} justifies the notion of variance suppression. 

If we further assume that the gradient is also bounded, meaning $\mathbb{E}[\| \nabla f(\mathbf{x}_t) \|^2] \leq G$, it is possible to further tighten the dependence on $T$ in Theorem \ref{thm.vso} as follows
\begin{align*}
	\mathbb{E} \big[ \| \mathbf{d}_t - \nabla f(\mathbf{x}_t) \|	^2 \big] \leq \theta \sigma^2 + {\cal O}\bigg(\frac{(1-\theta)^2 G}{\theta^2 T}\bigg).
\end{align*}
The proof follows directly by tightening inequality \eqref{apdx.eq.mse1} with the bounded gradient assumption. Other theorems in this paper also have an improved $T$ dependence given this extra  assumption, but we do not state them explicitly to avoid repetition. Next, we quantify the claimed stability with the suppressed variance. 
 
\begin{theorem}[Stabilizing adversaries with VASSO.]\label{thm.salad}
Under Assumptions \ref{as.1} -- \ref{as.3}, and the hyperparameter choices of Theorem \ref{thm.vso}, the stochastic linearization is $\rho\big[\sqrt{\theta}\sigma + {\cal O}(\frac{\sigma }{\theta T^{1/4}})\big]$-stable for VASSO, compared with $\rho\sigma$-stability of SAM. 
\end{theorem}

Theorem \ref{thm.salad} contends that VASSO alleviates the friendly adversary challenge by promoting stability. Qualitatively, VASSO is roughly $\sqrt{\theta} \in (0,1)$ times more stable relative to SAM, since the term in big ${\cal O}$ notation is negligible given a sufficiently large $T$. Theorem \ref{thm.salad} also guides the choice of $\theta$ -- preferably small but not too small, otherwise the term in big ${\cal O}$ is inversely amplified. 

\textbf{Visualize stabilized adversaries.} To validate our findings in Theorem \ref{thm.vso}, we train a ResNet-18 on CIFAR10 and depict the evolution of $\| \bm{\epsilon}_t -  \bm{\epsilon}_{t-1} \|$ for both SAM and VASSO in Fig. \ref{fig.vasso_eps}. Because both $\bm{\epsilon}_t$ and $\bm{\epsilon}_{t-1}$ reside on $\mathbb{S}_\rho(\mathbf{0})$, the value of $\| \bm{\epsilon}_t -  \bm{\epsilon}_{t-1} \|$ is solely determined by the angle between these vectors. In essence, $\| \bm{\epsilon}_t -  \bm{\epsilon}_{t-1} \|$ characterizes well the stability. The results presented in Fig. \ref{fig.vasso_eps} distinctly reveal that VASSO indeed improves the stability adversary over SAM.

\subsection{Additional perspectives of VASSO}

Having dealt with stability, this subsection proceeds with other aspects of VASSO.

\textbf{Convergence.} Summarized in the following corollary, the convergence of VASSO can be pursued as a direct consequence of Theorem \ref{thm.sam}. The reason is that $\bm{\epsilon}_t \in \mathbb{S}_\rho(\mathbf{0})$ is satisfied by \eqref{eq.vso}.  

\begin{corollary}[VASSO convergence]\label{coro.vso}
Under Assumptions \ref{as.1} -- \ref{as.3}, and with $\eta_t$ and $\rho$ as in Theorem \ref{thm.sam}, VASSO ensures for any $\theta \in (0,1)$ that 
	\begin{align*}
		& \frac{1}{T}\sum_{t=0}^{T-1}\mathbb{E}\big[ \| \nabla f(\mathbf{x}_t )\|^2 \big]	\leq {\cal O} \bigg( \frac{\sigma^2}{\sqrt{T}} \bigg) ~~~~\text{and}~~~~  \\
		& \frac{1}{T}\sum_{t=0}^{T-1}\mathbb{E}\big[ \| \nabla f(\mathbf{x}_t + \bm{\epsilon}_t)\|^2 \big]\leq {\cal O} \bigg( \frac{\sigma^2}{\sqrt{T}} \bigg).
	\end{align*}
\end{corollary}

\begin{figure}[t]
	\centering	
	\includegraphics[width=0.44\textwidth]{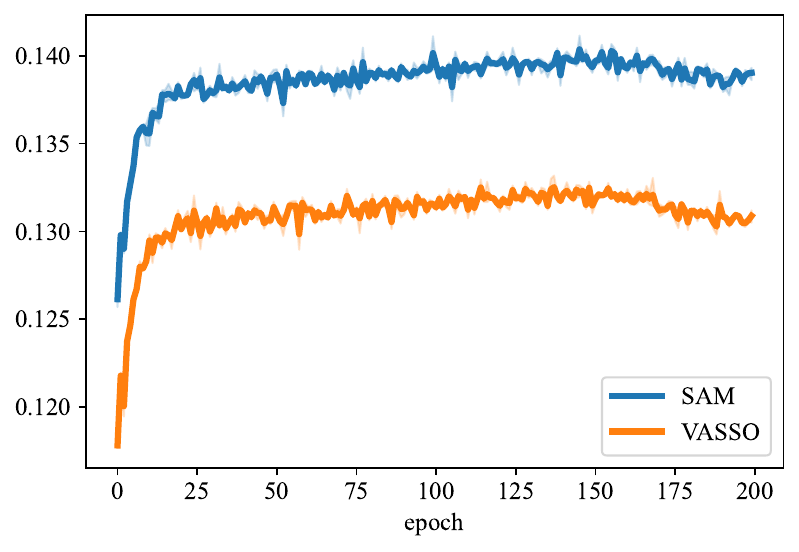}
	\vspace{-0.45cm}
	\caption{The adversary of VASSO is more stable than SAM.}
	\label{fig.vasso_eps}
\end{figure}

\textbf{Pronounced sharpness around optimum.} Consider a near optimal region, where $\| \nabla f(\mathbf{x}_t)\| \rightarrow 0$, and suppose a big data regime, where $\mathbf{g}_t(\mathbf{x}_t) = \nabla f(\mathbf{x}_t) + \bm{\zeta}$ for some Gaussian random vector $\bm{\zeta}$ with covariance matrix  $\sigma^2\mathbf{I}$ for simplicity. (Using arguments from von Mises-Fisher statistics \citep{mardia2000directional} generalization is also possible). SAM has difficulty estimating the flatness in this case, since $\bm{\epsilon}_t \approx \rho\bm{\zeta} / \|\bm{\zeta} \|$ is uniformly distributed over $\mathbb{S}_\rho(\mathbf{0})$ regardless of whether the neighboring region is sharp or not. On the other hand, VASSO has $\bm{\epsilon}_t = \rho\mathbf{d}_t / \|\mathbf{d}_t \|$. Because $\{\mathbf{g}_\tau (\mathbf{x}_\tau)\}_\tau$ on sharper valleys tend to have larger magnitude, their EMA $\mathbf{d}_t$ is helpful for distinguishing sharp valleys with flat ones.

\textbf{Memory efficient implementation.} Although at first glance VASSO has to store both $\mathbf{d}_t$ and $\bm{\epsilon}_t$, it can be implemented in a much more memory efficient manner. It is sufficient to store $\mathbf{d}_t$ together with a scalar $\| \mathbf{d}_t \|$ so that $\bm{\epsilon}_t$ can be recovered on demand through normalization; see \eqref{eq.vso_b}. Hence, VASSO's memory requirements are the same as those of SAM.

\textbf{Extensions.} VASSO has the potential to boost the performance of other SAM variants by stabilizing their stochastic linearization through variance suppression. For instance, adaptive SAM methods such as \citep{kwon2021,kim2022} are scale invarianct, while GSAM \citep{zhuang2022} minimizes a surrogate gap jointly with \eqref{eq.prob}. Nevertheless, these SAM variants leverage stochastic linearization in \eqref{eq.sam_epsilon_full}. It is thus envisioned that VASSO can also alleviate the possible friendly adversary issues therein. As we shall see in Sec. \ref{sec.numerical}, cross-fertilizing VASSO with GSAM yields improved numerical performance.  

\section{VASSO improves generalization for computation tradeoff}

\begin{algorithm}[t]
    \caption{Efficient VASSO (eVASSO)} \label{alg.r-vasso}
    \begin{algorithmic}[1]
    	\State \textbf{Initialize:} $\mathbf{x}_0, \rho, p$, and a sequence of iid Bernoulli random variables $\{ R_t \}_{t=0}^{T-1}$, where $R_t=1$ with probability $p$.
    	\For {$t=0,\dots,T-1$}
    		\State Sample a minibatch ${\cal B}_t$
    		\State Calculate stochastic gradient $\mathbf{g}_t(\mathbf{x}_t)$ on ${\cal B}_t$
    		\State Calculate $\mathbf{d}_t = (1 - \theta)	\mathbf{d}_{t-1} + \theta \mathbf{g}_t(\mathbf{x}_t)$
    		\If{$R_t = 1$} \Comment{\gray{Calculate the second gradient}}
    		\State $\bm{\epsilon}_t =  \argmax_{\| \bm{\epsilon} \| \leq \rho} f(\mathbf{x}_t) + \langle \mathbf{d}_t, \bm{\epsilon} \rangle = \rho \frac{\mathbf{d}_t}{ \| \mathbf{d}_t \| }$ 
    		\State Calculate stochastic gradient $\hat{\mathbf{g}}_t = \mathbf{g}_t(\mathbf{x}_t + \bm{\epsilon}_t)$ 
    		 \Else {~~$\hat{\mathbf{g}}_t = \mathbf{g}_t(\mathbf{x}_t)$} \Comment{\gray{Skip gradient computation}}
    		\EndIf
 
			\State Update model via $\mathbf{x}_{t+1} = \mathbf{x}_t - \eta \hat{\mathbf{g}}_t$
		\EndFor
		\State \textbf{Return:} $\mathbf{x}_T$
	\end{algorithmic}
\end{algorithm}

The generalization benefits of SAM come at the price of approximately doubled computation per iteration compared with SGD because of the need to compute $\mathbf{g}_t(\mathbf{x}_t)$ and $\mathbf{g}_t(\mathbf{x}_t + \bm{\epsilon}_t)$. While there are works trading off generalization for computation \citep{du2022,zhao222,jiang2023}, unfortunately these efforts are still under the threat of friendly adversaries, suggesting that the optimal tradeoff is elusive.

In pursuit of a more favorable tradeoff that leans toward the generalization side, our idea is to leverage VASSO's stabilized stochastic linearization in eSAM proposed originally in \citep{zhao222}. The resultant algorithm termed as efficient VASSO (eVASSO) is listed as Alg. \ref{alg.r-vasso}. 
It introduces a Bernoulli random variable $R_t$ per iteration, where $R_t=1$ with probability $p \in [0,1]$. eVASSO alternates between VASSO or SGD based on the value of $R_t$ to save computation. It can be seen that the number of gradient computation is $(1+p)$ per iteration in expectation. The hyperparameter $p$ plays a pivotal role in determining the generalization-computation tradeoff of eVASSO. 
A smaller $p$ advocates more aggressive reduction of computational time, while heightening the risk of compromised generalization comparing with vanilla VASSO. However, thanks to the stabilized stochastic linearization, eVASSO still outperforms SAM for $R_t=1$, as asserted next. 

\begin{theorem}[Variance suppression and stability of eVASSO]\label{thm.r-vso}
Under Assumptions \ref{as.1} -- \ref{as.3},  if i) $\bm{\epsilon}_t$ is obtained by \eqref{eq.vso} with $\theta \in (0, 1)$; and, ii) $\eta_t$ and $\rho$ are selected as in Theorem \ref{thm.sam}, then eVASSO guarantees that for any $p \in (0, 1]$, it holds that 
	\begin{align*}
		\mathbb{E} \big[ \| \mathbf{d}_t - \nabla f(\mathbf{x}_t) \|	^2 \big] \leq \theta \sigma^2 + {\cal O}\bigg(\frac{(1-\theta)^2 \sigma^2}{\theta^2 \sqrt{T}}\bigg).
	\end{align*}
	Moreover, whenever $R_t=1$, the stochastic linearization of eVASSO is $\rho\big[\sqrt{\theta}\sigma + {\cal O}(\frac{\sigma }{\theta T^{1/4}})\big]$-stable.
\end{theorem}

Note that Theorem \ref{thm.r-vso} is the best one can achieve since there is no need for stochastic linearization when $R_t=0$.

\begin{figure}
\centering	
\includegraphics[width=0.47\textwidth]{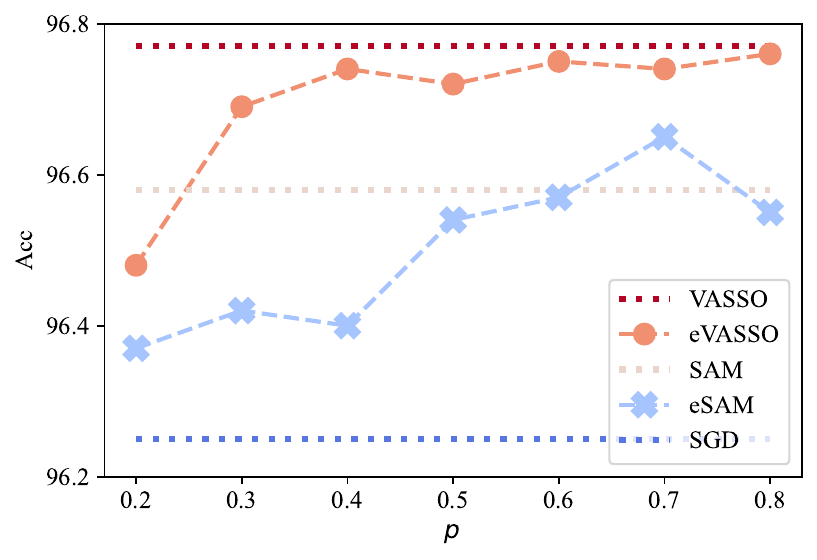}
\vspace{-0.45cm}
\caption{eVASSO improves generalization-computation tradeoff.}
\label{fig.evasso}
\end{figure}

To visualize the generalization-computation tradeoff in eVASSO relative to eSAM, a ResNet-18 is trained on CIFAR-10 with $p$ chosen from $0.2$ to $0.8$. The test accuracy is plotted in Fig. \ref{fig.evasso}. It can be seen that eSAM saves computation at the price of compromised test accuracy relative to SAM. However, eVASSO effectively steers the tradeoff towards the desirable side -- less computation (small $p$) with higher accuracy. eVASSO with $p=0.3$ has already surpassed the performance of vanilla SAM and eSAM with $p=0.7$. This means that for the same test accuracy, eVaSSO can reduce the computational overhead by $70\%$ and $57\%$ relative to SAM and eSAM, respectively. 

\begin{table*}[t]
	\centering
	\normalsize
	\caption{Test accuracy (\%) comparison on various neural networks trained on CIFAR10.}
    \vspace{-.2cm}
	\begin{tabular}{l|cccccc}
	\toprule
	CIFAR10            & SGD                & SAM                & ASAM               & FisherSAM & F-SAM  & VASSO         
	 \\
	\midrule
	\textbf{VGG-11-BN}     & 93.20$_{\pm0.05}$  & 93.82$_{\pm0.05}$  & 93.47$_{\pm0.04}$  & 93.60$_{\pm0.09}$   &   93.94$_{\pm 0.09}$  & \textbf{94.10$_{\pm0.07}$}  \\
	\textbf{ResNet-18}   & 96.25$_{\pm0.06}$  & 96.58$_{\pm0.10}$  & 96.33$_{\pm0.09}$  & 96.72$_{\pm0.03}$  &   96.72$_{\pm 0.06}$   &\textbf{96.77$_{\pm0.09}$}  \\
	\textbf{WRN-28-10}   & 97.08$_{\pm0.16}$  & 97.32$_{\pm0.11}$  & 97.15$_{\pm0.05}$  & 97.46$_{\pm0.18}$ &    97.40$_{\pm 0.09}$ &\textbf{97.54$_{\pm0.12}$}  \\
	\textbf{PyramidNet-110}   & 97.39$_{\pm0.09}$  & 97.85$_{\pm0.14}$  & 97.56$_{\pm0.11}$  & 97.84$_{\pm0.18}$   &   97.72$_{\pm 0.11}$   &\textbf{97.93$_{\pm0.08}$}  \\
	\bottomrule
	\end{tabular}
	\vspace{-0.1cm}
	\label{tab.cifar10}
\end{table*}

\begin{table*}[t]
	\centering
	\normalsize
	\caption{Test accuracy (\%) comparison on various neural networks trained on CIFAR100.}
    \vspace{-.2cm}
	\begin{tabular}{l|cccccc}
	\toprule
	CIFAR100             & SGD                & SAM                & ASAM               & FisherSAM    & F-SAM  & VASSO                       \\
	\midrule 
	\textbf{ResNet-18}   & 77.90$_{\pm0.07}$  & 80.96$_{\pm0.12}$  & 79.91$_{\pm0.04}$  & 80.99$_{\pm0.13}$  & 80.58$_{\pm 0.04}$  &\textbf{81.30}$_{\pm0.13}$  \\
	\textbf{WRN-28-10}   & 81.71$_{\pm0.13}$  & 84.88$_{\pm0.10}$  & 83.54$_{\pm0.14}$  & 84.91$_{\pm0.07}$  & 84.62$_{\pm 0.15}$  & \textbf{85.06}$_{\pm0.05}$  \\
	\textbf{PyramidNet-110}   & 83.50$_{\pm0.12}$  & 85.60$_{\pm0.11}$  & 83.72$_{\pm0.09}$  & 85.55$_{\pm0.14}$  & \textbf{86.03}$_{\pm 0.20}$  & 85.85$_{\pm0.09}$  \\
	\bottomrule
	\end{tabular}
	\vspace{-0.1cm}
	\label{tab.cifar100}
\end{table*}

\begin{table*}[t]
	\centering
	\normalsize
	\caption{Test accuracy (\%) of VASSO on ImageNet, where V+G is short for VASSO + GSAM.}
    \vspace{-.2cm}
	\begin{tabular}{l|ccccccc}
	\toprule
	          ~ImageNet            & vanilla      & SAM              & ASAM     & GSAM     & F-SAM  & VASSO      &  V+G                   \\
	\midrule 
	\textbf{ResNet-50}   & 76.62$_{\pm0.12}$  & 77.16$_{\pm0.14}$ & 77.10$_{\pm0.16}$  &  77.20$_{\pm0.13}$  & 77.15 $_{\pm 0.12}$  & \textbf{77.42}$_{\pm0.13}$  & \textbf{77.48}$_{\pm0.04}$ \\
	\textbf{ViT-S/32}    & 68.12$_{\pm0.05}$  & 68.98$_{\pm0.08}$ &  68.74$_{\pm0.11}$  &   69.42$_{\pm0.18}$  & 69.43 $_{\pm 0.10}$  & \textbf{69.54}$_{\pm0.15}$  & \textbf{69.61}$_{\pm0.11}$ \\
	\bottomrule
	\end{tabular}
	\vspace{-0.1cm}
	\label{tab.imagenet}
\end{table*}

\begin{table}[t]
	\centering
	\caption{Leave-one-out cross-validation accuracy (\%) on DomainBed.}
    \vspace{-.2cm}
	\normalsize
	\begin{tabular}{l|cccc}
	\toprule
	\specialrule{0em}{2pt}{2pt}
	DomainBed  & Adam             & SAM              & GSAM               & VASSO    \\
	\midrule 
 	PACS       & 85.5$_{\pm0.2}$  & 85.8$_{\pm0.2}$  & 85.9$_{\pm0.1}$  &  \textbf{86.0}$_{\pm0.1}$ \\
 	VLCS       & 77.3$_{\pm0.4}$  & 79.4$_{\pm0.1}$  & 79.1$_{\pm0.2}$           &  \textbf{79.6}$_{\pm0.2}$  \\
	OfficeHome & 66.5$_{\pm0.3}$  & 69.6$_{\pm0.1}$  & 69.3$_{\pm0.0}$    &  \textbf{69.8}$_{\pm0.2}$ \\
	TerraInc   & 46.1$_{\pm1.8}$  & 43.3$_{\pm0.7}$  & \textbf{47.0}$_{\pm0.8}$    &  \textbf{47.0}$_{\pm0.3}$ \\
	\midrule 
	Average   & 68.9  & 69.5  & 70.3  & \textbf{70.6}  \\
	\bottomrule
	\end{tabular}
	\vspace{-0.1cm}
 	\label{tab.domain}
\end{table}

\section{Numerical tests}\label{sec.numerical}

To support our theoretical findings and validate how powerful variance suppression is, this section assesses generalization performance of VASSO and eVASSO on various learning tasks across vision and language domains. All experiments are run on NVIDIA V100 and H100 GPUs. Online code that is used is available at \url{https://github.com/BingcongLi/VaSSO}.
 
\subsection{Image classification}\label{sec.num.img}

\textbf{CIFAR10.} Building on the selected base optimizers such as SGD, the test accuracy of VASSO is compared with SAM and two adaptive approaches, namely ASAM and FisherSAM \citep{foret2021,kwon2021,kim2022}. In addition, F-SAM~\cite{F-SAM} is included as a closely related baseline.
Neural networks including VGG-11, ResNet-18, WRN-28-10 and PyramidNet-110 are trained on CIFAR10. Standard implementation including random crop, random horizontal flip, normalization and cutout \citep{cutout2017} are leveraged for data augmentation. The first three models are trained for $200$ epochs with a batch size of $128$, and PyramidNet-110 is trained for $300$ epochs using batch size $256$. The cosine learning rate schedule is applied in all settings. The first three models use an initial learning rate of $0.05$, and PyramidNet adopts $0.1$. Weight decay is chosen as $0.001$ for SAM, ASAM, FisherSAM, F-SAM and VASSO, but $0.0005$ for SGD following \citep{du2022,mi2022}. We tune $\rho$ across the values $\{0.01, 0.05, 0.1, 0.2, 0.5 \}$ for SAM; $\rho=0.1$ yields best results for ResNet and WRN; while $\rho=0.05$ and $\rho=0.2$ suit best for VGG and PyramidNet, respectively. ASAM, F-SAM and VASSO adopt the same $\rho$ as SAM. FisherSAM uses the recommended $\rho=0.1$ \citep{kim2022}. F-SAM adopts the suggested momentum factor $\lambda = 0.9$ for WRN-28-10 and PyramidNet-110, and $\lambda=0.6$ for other networks~\cite{F-SAM}. For VASSO, we tune $\theta =\{ 0.4, 0.9 \}$, and report the best accuracy. We find that $\theta=0.4$ works best for ResNet-18 and WRN-28-10, while $\theta=0.9$ achieves the best accuracy in other cases. 

Table \ref{tab.cifar10} shows that VASSO offers $0.2$ to $0.3$ accuracy improvement over SAM in all tested scenarios except for PyramidNet-110, where the improvement is about $0.1$. These results corroborate that generalizability indeed benefits from variance suppression and the induced stabilized adversary.

\textbf{CIFAR100.} The training setups on this dataset are the same as those on CIFAR10, except that for SAM the best choice is $\rho=0.2$. The numerical results are listed in Table \ref{tab.cifar100}. It can be seen that SAM gains considerably in generalization over SGD, and this gain is further amplified by VASSO. On all tested models, VASSO improves the test accuracy of SAM by $0.2$ to $0.3$. These experiments once again confirm the improved generalization of VASSO thanks to the stabilized adversary.

\textbf{ImageNet.} Next, we investigate the performance of VASSO on larger scale experiments by training ResNet-50 and ViT-S/32 on ImageNet \citep{imagenet2009}. Implementation details are deferred to Appendix \ref{apdx.sec.numerical}. Note that the baseline optimizer is SGD for ResNet and AdamW for ViT \citep{kingma2014, loshchilov2017adamw}. VASSO is also integrated with GSAM (V+G) to demonstrate that variance suppression also benefits other SAM variants such as the one in \citep{zhuang2022}. For ResNet-50, it can be observed that vanilla VASSO outperforms other SAM variants, and offers a gain of $0.26$ over SAM. V+G showcases the best performance with a gain of $0.28$ on top of GSAM. VASSO and V+G also exhibit the best test accuracy on ViT-S/32, where VASSO improves SAM by 0.56 and V+G outperforms GSAM by 0.19. These numerical improvement demonstrates that stability of adversaries is indeed desirable.

\begin{table}[t]
	\centering
	\normalsize
	\caption{Test accuracy (\%) on CIFAR10 under different levels of label noise.}
    \vspace{-.2cm}
	\begin{tabular}{l|cccc}
	\toprule
	  Label   & \multirow{2}{*}{SAM}   & VASSO            & VASSO                  & VASSO\\
	  noise                &                        & ($\theta=0.9$)   & ($\theta=0.4$)        & ($\theta=0.2$)
	 \\
	\midrule
	\textbf{25\%}        & 96.39$_{\pm0.12}$  & 96.36$_{\pm0.11}$    & 96.42$_{\pm0.12}$    		 & \textbf{96.48}$_{\pm0.09}$ \\
	\textbf{50\%}        & 93.93$_{\pm0.21}$  & 94.00$_{\pm0.24}$    & 94.63$_{\pm0.21}$    & \textbf{94.93}$_{\pm0.16}$ \\
	\textbf{75\%}        & 75.36$_{\pm0.42}$  & 77.40$_{\pm0.37}$    & 80.94$_{\pm0.40}$    & \textbf{85.02}$_{\pm0.39}$ \\
	\bottomrule
	\end{tabular}
	\vspace{-0.1cm}
	\label{tab.lbnoise}
\end{table}

\begin{table*}[t]
	\centering
	\normalsize
	\caption{Performance of VASSO for training a Transformer on IWSLT-14 dataset.}
    \vspace{-.2cm}
	\tabcolsep=0.33cm
	\begin{tabular}{l|ccccc}
	\toprule
	                           & \multirow{2}{*}{AdamW}      & \multirow{2}{*}{SAM}  & \multirow{2}{*}{ASAM} & VASSO             & VASSO  \\
	                           &                           &                       &                       & ($\theta=0.9$)    & ($\theta=0.4$)
	 \\
	\midrule
	Val. ppl. & 5.02$_{\pm0.03}$    & 5.00$_{\pm0.04}$     & \textbf{4.99}$_{\pm0.03}$   & 5.00$_{\pm0.03}$  & \textbf{4.99}$_{\pm0.03}$ \\
	BLEU      & 34.66$_{\pm0.06}$    & 34.75$_{\pm0.04}$      & 34.76$_{\pm0.04}$    & 34.81$_{\pm0.04}$   &  \textbf{34.88}$_{\pm0.03}$  \\
	\bottomrule
	\end{tabular}
	\vspace{-0.1cm}
	\label{tab.nmt}
\end{table*}

\subsection{Domain generalization}\label{sec.dg}

This subsection evaluates VASSO's efficiency on domain generalization (DG) tasks. 
The goal of DG is to facilitate seamless transfer of learned knowledge from a ``source domain'' to generalize well on unseen yet related ``target domains'' \cite{zhou2022domain,oza2023}. The major challenge is the distributional shift across these domains, and recent studies have demonstrated that a flat minimum is beneficial to the desirable generalizability  \citep{wang2023}. 

Numerical tests of VASSO are conducted on DomainBed, which is a well-established benchmark \citep{gulrajani2022}. We focus on 4 specific datasets: PACS \citep{Li_2017_ICCV}, VLCS \citep{Fang_2013_ICCV}, OfficeHome \citep{Venkateswara_2017_CVPR}, and TerraInconita \citep{Beery_2018_ECCV}; see Table \ref{tab.domain-summary} for details. The performance is evaluated using standard leave-one-out cross-validation \citep{gulrajani2022}, and this experiment is run on a NVIDIA A5000 GPU.

Following \citep{gulrajani2022}, a ResNet-50 imported from pytorch and pretrained on ImageNet is adopted as the backbone. We use a pretrained checkpoint on purpose -- to demonstrate that VASSO can enhance the generalization even for finetuning. Adam is adopted as the base optimizer, and hyperparamters are chosen the same as in \citep{gulrajani2022,Wang_2023_CVPR}. For VASSO, $\theta$ takes values from $\{0.2, 0.4, 0.9\}$.

Table \ref{tab.domain} compares VASSO with Adam, SAM, and GSAM, where the results of the last three are obtained from \citep{Wang_2023_CVPR}. Detailed per-domain results can be found in Appendix \ref{apdx.sec.domainbed}. It is observed that VASSO surpasses all these competitors in terms of averaged accuracy, underscoring the benefit of variance suppression in the scenario of distributional shift across diverse domains.

\subsection{Label noise}

It is known that SAM holds great potential to harness robustness under the appearance of label noise in training data \citep{foret2021,jiang2023dyn}. As the training loss landscape is largely perturbed by the label noise, this is a setting where the suppressed variance and stabilized adversaries are expected to improve performance. In our experiments, VASSO's performance is evaluated when a fraction of the training labels are randomly flipped. With $\theta$ taking $\{ 0.9, 0.4, 0.2 \}$ values, the corresponding test accuracies are listed in Table \ref{tab.lbnoise}.

Our first observation is that VASSO outperforms SAM at all three levels of label noise tested. VASSO elevates generalization improvement as the percentage of noisy labels grows. For $75\%$ noisy labels,  VASSO with $\theta=0.4$ outperforms SAM with an absolute improvement of $5.58$, while VASSO with $\theta=0.2$ markedly improves SAM by $9.66$. In all setups, $\theta=0.2$ leads to best performance while $\theta=0.9$ exhibits the worst generalization when compared with VASSO. When fixing the chosen $\theta$, e.g., $\theta=0.2$, VASSO has larger absolute accuracy improvement over SAM under higher level of label noise. These observations are in agreement with Theorem \ref{thm.salad}, which predicts that VASSO is suitable for settings with larger label noise due to its enhanced stability, especially when $\theta$ is chosen small (but not too small).

\begin{table}[t]
	\centering
	\normalsize
	\caption{Test accuracy on GSM8K and pass@1 (\%) on HumanEval.}
    \vspace{-.2cm}
	\tabcolsep=0.4cm
	\begin{tabular}{l|cc}
	\toprule
	  Method   & GSM8K      & HumanEval\\
	\midrule
	FullFT & 59.74$_{\pm 0.69}$ & 33.12$_{\pm 0.32}$ \\
    LoRA  &  58.21$_{\pm 0.34}$ & 24.75$_{\pm 0.23}$ \\
    \midrule
    LoRA-SAM & 59.16$_{\pm 0.52}$ & 26.59$_{\pm 0.36}$ \\
    LoRA-VASSO  &  \textbf{61.81}$_{\pm 0.01}$  &  \textbf{33.74}$_{\pm 0.58}$ \\
	\bottomrule
	\end{tabular}
	\vspace{-0.1cm}
	\label{tab.nlg}
\end{table}

\subsection{Neural machine translation}

Having demonstrated the benefits of variance suppression on vision tasks, here we will test VASSO on German to English translation using a Transformer \citep{vaswani2017attention} trained on IWSLT-14 dataset \citep{iwslt2014}, and the fairseq implementation. AdamW is chosen as the base optimizer in SAM and VASSO because of its improved performance over SGD. The learning rate of AdamW is initialized to $5\times 10^{-4}$, and then follows an inverse square root schedule. For momentum, we choose $\beta_1=0.9$ and $\beta_2=0.98$. Label smoothing is also applied with a rate of $0.1$. Hyperparameter $\rho$ is tuned for SAM from $\{ 0.01, 0.05, 0.1, 0.2\}$, and $\rho=0.1$ performs the best. The same $\rho$ is picked for ASAM and VASSO.

The validation perplexity and test BLEU scores are listed in Table \ref{tab.nmt}. It can be seen that both SAM and ASAM have better performance on validation perplexity and BLEU relative to AdamW. Although VASSO with $\theta=0.9$ has slightly higher validation perplexity, its BLEU score outperforms SAM and ASAM. VASSO with $\theta=0.4$ showcases the best generalization performance on this task, providing a $0.22$ improvement on BLEU score relative to AdamW. This aligns with Theorems \ref{thm.vso} and \ref{thm.salad}, which suggest that a small $\theta$ is more beneficial to the stability of adversaries.

\subsection{Natural language generation}
Beyond translation, SAM has recently been adopted in parameter-efficient fine-tuning of large language models (LLMs) to enhance generalization on downstream applications~\cite{li2024implicit}. We next evaluate VASSO on mathematical reasoning and code generating tasks. Specifically, we fine-tune LLaMA 2 7B models~\cite{touvron2023llama} with low-rank adaptation (LoRA)~\cite{hu2022lora} on 100k subsets of MetaMathQA~\cite{ICLR2024_c400474e} and Code-Feedback~\cite{zheng2024opencodeinterpreter}, and evaluate the resultant models on GSM8K~\cite{cobbe2021gsm8k} and HumanEval~\cite{chen2021evaluating} benchmarks, Mathematical reasoning performance is assessed using accuracy, whereas code generation is evaluated via the pass@1 metric. Following the setups in~\cite{liu2026bilora}, we use a LoRA rank of $r=8$, and AdamW optimizer with $5 \times 10^{-4}$ learning rate, $3\%$ warmup iterations, cosine scheduler, and macro/micro batch size of 32/4 for 2 epochs. VASSO directly uses the perturbation radius $\rho = 0.01$ tuned for SAM, and sets $\theta = 0.4$ for both tasks. As the micro batch size is as small as 4, the adversarial perturbation can exhibit remarkable variance and become overly friendly to other batches, including those within the same macro batch. 

Table~\ref{tab.nlg} compares LoRA-VASSO with LoRA-SAM, as well as the sharpness-unaware baselines FullFT and LoRA. The results for FullFT, LoRA, and LoRA-SAM are taken directly from~\cite{liu2026bilora}. While FullFT can overfit due to the large model size, LoRA restricts the number of trainable parameters at the expense of reduced generalization. LoRA-SAM provides only marginal improvement over LoRA and still generalizes worse than FullFT. In contrast, LoRA-VASSO consistently delivers stronger performance on both tasks despite the small micro batch size, demonstrating the benefit of stabilizing adversarial perturbations through variance suppression.

\begin{table*}[t]
    \centering
    \caption{Test accuracy and training runtime comparison on CIFAR-10 and CIFAR-100 datasets. Numbers in parentheses indicate overhead relative to SGD.}
    \vspace{-.2cm}
    \tabcolsep=0.34cm
   \normalsize
    \begin{tabular}{l|cl|cc}
    \toprule
         & \multicolumn{2}{c|}{\textbf{ResNet-18}} & \multicolumn{2}{c}{ \textbf{WRN-28-10}} \\
\midrule
     CIFAR-10 &   Accuracy & epoch runtime (s)  & Accuracy & epoch runtime (s) \\
            \midrule
        SGD & 96.25$_{\pm 0.05}$  & 14.42 (1x)       & 97.08$_{\pm 0.16}$ & 80.58 (1x) \\
        SAM & 96.58$_{\pm 0.10}$  & 27.62 (1.92x)    & 97.32$_{\pm 0.11}$ & 156.92 (1.95x) \\
        eSAM ($p=0.5$) & 96.54$_{\pm 0.11}$    & 21.57 (1.49x) & 97.21$_{\pm 0.08}$  & 119.06  (1.48x)  \\
        \rowcolor{lightgray}
        eVASSO ($p=0.5$) & 96.72$_{\pm 0.14}$  & 21.94 (1.52x)    & 97.36$_{\pm 0.11}$  & 120.11  (1.49x)  \\
        eSAM ($p=0.3$) & 96.42$_{\pm 0.08}$ & 18.34 (1.27x)       & 97.20$_{\pm 0.11}$ & 105.63  (1.31x)  \\
        \rowcolor{lightgray}
        eVASSO ($p=0.3$) & 96.69$_{\pm 0.06}$   & 18.61 ((1.29x)   & 97.29$_{\pm 0.16}$ & 105.99  (1.32x)  \\
	\midrule
     CIFAR-100 &   Accuracy & epoch runtime (s)  & Accuracy & epoch runtime (s) \\
            \midrule
        SGD & 77.90$_{\pm 0.07}$  & 14.87 (1x)       & 81.71$_{\pm 0.13}$ & 81.02 (1x) \\
        SAM & 80.96$_{\pm 0.12}$  & 28.37 (1.91x)    & 84.88$_{\pm 0.10}$  & 157.44 (1.94x) \\
        eSAM ($p=0.5$) & 81.03$_{\pm 0.17}$   & 21.77 (1.46x)     & 84.31$_{\pm 0.15}$ & 120.28 (1.48x)  \\
        \rowcolor{lightgray}
        eVASSO ($p=0.5$) & 81.20$_{\pm 0.15}$  & 22.09 (1.49x)    & 84.52$_{\pm 0.06}$ & 121.55  (1.49x)  \\
        eSAM ($p=0.3$) & 80.60$_{\pm 0.12}$ & 19.32 (1.30x)       & 83.65$_{\pm 0.18}$ & 105.41 (1.30x)  \\
        \rowcolor{lightgray}
        eVASSO ($p=0.3$) & 80.84$_{\pm 0.06}$  & 19.35 ((1.30x)   & 84.07$_{\pm 0.08}$  & 106.61 (1.30x) \\        \bottomrule
    \end{tabular}
    \label{tab.rvasso}
   \vspace{-0.1cm}
\end{table*}

\subsection{Generalization-computation tradeoff}\label{sec.evasso}


As previously illustrated in Fig. \ref{fig.evasso}, eVASSO improves the generalization-computation tradeoff. This subsection offers further supporting evidence to this end. In particular, the proposed eVASSO with $p=0.5$ and $p=0.3$ is tested on ResNet-18 and WRN-28-10 using datasets CIFAR10 and CIFAR100. The test accuracy vs runtime of eVASSO, and benchmark algorithms such as SGD, SAM and eSAM can be found in Table \ref{tab.rvasso}. Note that for eVASSO, we tune $\theta$ from $\{0.4 , 0.9 \}$, and report the best results.


As expected, the excess runtime of eVASSO and eSAM relative to SGD is roughly proportional to $p$. Another common observation is that a smaller $p$ degrades generalization, demonstrating the tradeoff pattern.

\textbf{CIFAR10.} Compared with SAM, eVASSO achieves better and comparable test accuracy on ResNet-18 and WRN-28-10, respectively, while requiring significantly reduced training time. On the other hand, eVASSO utilizes similar computational resources to reach markedly improved test accuracy over eSAM, because of the suppressed variance for finding adversaries. 

\textbf{CIFAR100.} On ResNet-18, the test accuracy of eVASSO with $p=0.5$ outperforms SAM, and with $p=0.3$ it is comparable to SAM. While on WRN-28-10, although eVASSO cannot match with SAM, it significantly outperforms eSAM and SGD. This suggests that even with minimally increased computation time (e.g., 1.3x), it is possible to improve the generalization over SGD drastically by $2.3$.

\subsection{Additional tests}

This subsection evaluates several flatness-related metrics using a ResNet-18 trained on CIFAR10, where the hyperparameters are chosen as those in Sections \ref{sec.num.img} and \ref{sec.evasso}. 

\begin{table}[t]
	\centering
	\normalsize
	\caption{Hessian spectrum of a ResNet-18.}
    \vspace{-.2cm}
	\tabcolsep=0.34cm
	\begin{tabular}{l|cccc}
	\toprule
	                        & SGD    & SAM    & VASSO         & eVASSO               \\
	\midrule 
	$\lambda_1$          & 82.52  & 26.40  & \textbf{23.32}   &  25.21  \\
	$\lambda_1/\lambda_5$   & 16.63 &  2.12   & \textbf{1.86} &  2.03  \\
	\bottomrule
	\end{tabular}
	\vspace{-0.1cm}
	\label{tab.spectra}
\end{table}

\textbf{Hessian spectrum.}
We focus on the largest eigenvalue $\lambda_1$ and the ratio of the largest to the fifth largest eigenvalue $\lambda_1/\lambda_5$. These metrics are also adopted by \citep{foret2021,stanislaw2020} to reflect the flatness of the solution, where smaller numbers are more preferable. Because exact calculation of the Hessian spectrum is prohibitive given the ResNet-18 size, we instead leverage Lanczos' approximation algorithm~\citep{behrooz2019}. Table \ref{tab.spectra} shows that SAM indeed converges to a much flatter solution compared with SGD, and VASSO further improves upon SAM. This confirms that the friendly adversary effect is indeed alleviated by variance suppression, which in turn boosts the generalization of ResNet18 as shown earlier in Section \ref{sec.num.img}. Moreover, the flatness measure of eVASSO with $p=0.5$ is slightly worse compared with VASSO. This is a consequence of the computation-generalization tradeoff that is present in eVASSO.

\begin{figure}[t]
	\centering
	\begin{tabular}{cc}
	\vspace{-0.3cm}
		\hspace{-0.2cm}
		\includegraphics[width=.22\textwidth]{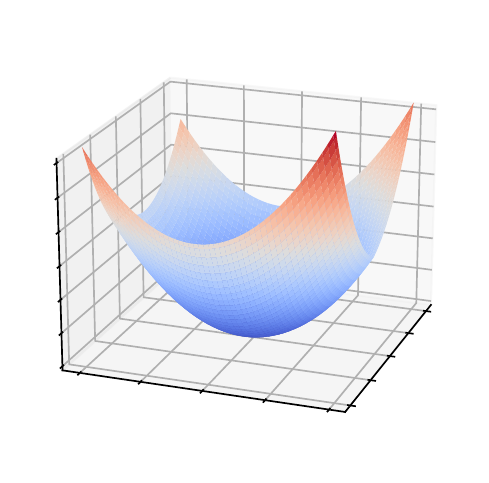}&
		\includegraphics[width=.22\textwidth]{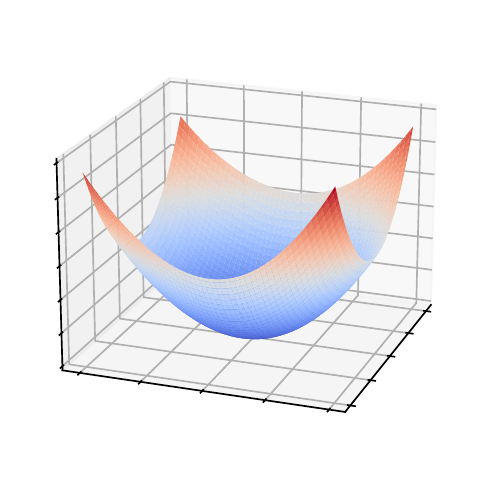}\\
		(a) SGD &  (b) SAM \\
		\hspace{-0.2cm}
		\includegraphics[width=.22\textwidth]{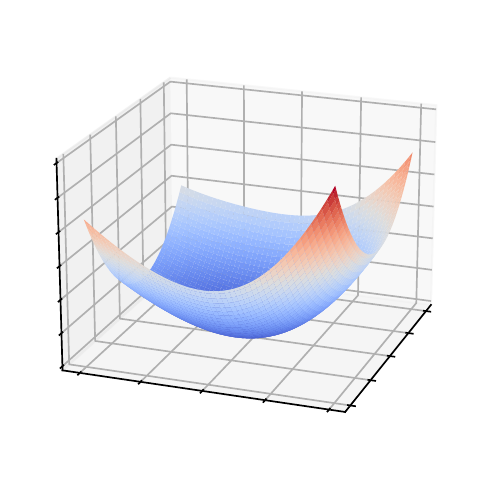}&
		\includegraphics[width=.22\textwidth]{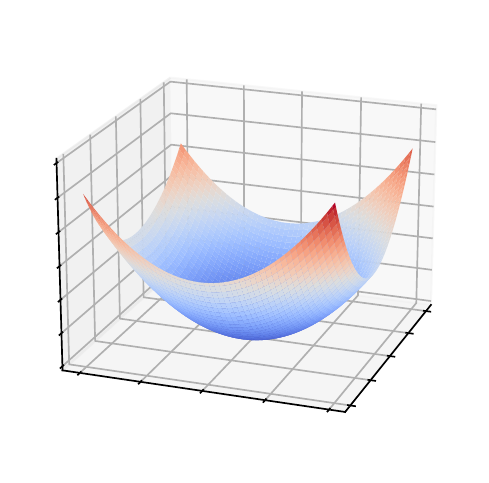}
		\\ 
		\hspace{-0.15cm}
	   (c) VASSO & (d) eVASSO
	\end{tabular}
	\caption{Visualization of loss landscapes.}
	 \label{fig.landscape}
	 \vspace{-0.3cm}
\end{figure}

\textbf{Landscape.} The loss landscapes optimized by the proposed approaches are plotted in Fig. \ref{fig.landscape} following the visualization method in \cite{li2018visual}. It can be seen that VASSO and eVASSO ($p=0.5$) indeed find flatter valleys relative to SAM and SGD, in agreement with their improved generalization behavior. 

\section{Other related works}
This section completes the context of DNN generalizability with related works.

\textbf{Sharpness and generalization.} Since the study of \cite{keskar2016}, the relation between sharpness and generalization has been investigated intensively. It has been observed that sharpness is closely correlated with the ratio between learning rate and batch size in SGD \citep{Jastrzebski2018}. Analytical results on the generalization error using sharpness-related measures can be found in e.g., \citep{dziugaite2017,behnam2017,wang2022}. These works justify the goal of seeking a flatter valley to enhance generalizability. Aiming at a flatter minimum, approaches other than SAM have been also developed. For example, stochastic weight averaging has been advocated for DNNs~\cite{izmailov2018}, while \cite{wu2020} has put forth an algorithm similar to SAM and has placed emphasis on robustness to adversarial training.

\textbf{SAM variants.} Besides GSAM and ASAM that have been already mentioned,  \citep{yang2022} pursued a SAM variant by penalizing the gradient norm on the premise that sharper valleys tend to have gradients with larger norm. In addition, \cite{barrett2020implicit} arrived at a similar conclusion by analyzing gradient flow. Exploiting multiple (ascent) steps to find an adversary has been studied systematically~\citep{kim2023multi}. However, these works overlook the friendly adversary issue, while VASSO provides algorithmic possibilities for generalization benefits by stabilizing their adversaries. 

\section{Concluding remarks}

This contribution demonstrates that stabilizing adversaries through variance suppression consolidates the generalization merits of sharpness-aware minimization. The novel framework abbreviated as VASSO, induces provable stability for SAM's stochastic linearization. VASSO's theoretically established merits are demonstrated with a suit of numerical experiments, and illustrate model-agnostic improvement over SAM on various vision and language tasks. Lastly, VASSO also steers the generalization-computation tradeoff of SAM towards the optimal side -- improved generalization (over SGD) with significantly reduced computation (relative to SAM). 

\bibliographystyle{IEEEtranN}
\bibliography{myabrv,datactr}

\newpage
\appendices

\section{SFW vis-a-vis SAM adversary}\label{apdx.sec.1}

The stochastic Frank-Wolfe (SFW) algorithm outlined in Alg. \ref{alg.sfw} solves the following generally nonconvex stochastic optimization
\begin{align}\label{eq.apdx.prob}
	\max_{\mathbf{x} \in {\cal X}} h(\mathbf{x}) := \mathbb{E}_\xi \big[ h(\mathbf{x}, \xi) \big]	
\end{align}
where ${\cal X}$ is a convex and compact constraint set. 

\begin{algorithm}[H]
    \caption{SFW \citep{reddi2016}}\label{alg.sfw}
    \begin{algorithmic}[1]
    	\State \textbf{Initialize:} $\mathbf{x}_0\in {\cal X}$
    	\For {$t=0,1,\dots,T-1$}
    		\State draw iid samples $\{ \xi_t^b\}_{b=1}^{B_t}$
    		\State let $\hat{\mathbf{g}}_t = \frac{1}{B_t} \sum_{b=1}^{B_t} \nabla h(\mathbf{x}_t, \xi_t^b)$
    		\State $\mathbf{v}_{t+1} = \argmax_{\mathbf{v} \in \cal X} \langle \hat{\mathbf{g}}_t , \mathbf{v} \rangle$
			\State $\mathbf{x}_{t+1} = (1-\gamma_t) \mathbf{x}_t + \gamma_t \mathbf{v}_{t+1}$ 
		\EndFor
	\end{algorithmic}
\end{algorithm}

The SFW iteration is convergent so long as the batch size $B_t = {\cal O}(T), \forall t$ is sufficiently large~\citep[Theorem 2]{reddi2016}. This is because line 5 in Alg. \ref{alg.sfw} is extremely sensitive to gradient noise. 

\subsection{The adversary of SAM}

By choosing $h(\bm{\epsilon}) = f(\mathbf{x}_t + \bm{\epsilon})$ and ${\cal X}= \mathbb{S}_\rho(\mathbf{0})$, it is not hard to show that 1-iteration SFW with $\gamma_0=1$ yields a solution equivalent to the stochastic linearization in SAM; cf. \eqref{eq.sam_epsilon_full} and \eqref{eq.sam_epsilon}. This link implies that the SAM adversary also suffers from stability issues in the same way as SFW does. Moreover, what amplifies this issue in SAM is the adoption of a constant batch size, which is typically small and far less than the ${\cal O}(T)$ required for SFW convergence.

Our VASSO solver is inspired by modified SFW approaches that leverage a constant batch size to ensure convergence; see e.g., \citep{mokhtari2018,zhang2021,li2021heavy}. Even though, coping with SAM's instability is still challenging in two major aspects. First, SAM uses \textit{one-step} SFW, which internally breaks the nice analytical structure. Moreover, the inner maximization (i.e., the objective function of SFW) \textit{varies from iteration to iteration} along with the updated $\mathbf{x}_t$. This link also suggests the potential of applying other FW approaches for the adversary such as those in  \citep{tsiligkaridis2022, li2020,li2020extra,huang2020}. We leave this direction for future research. 

\subsection{The three dimensional example in Fig. \ref{fig.noise}}\label{eq.apdx-3d}
Here we detail the implementation used for generating Fig. \ref{fig.noise}. We use $\nabla f(\mathbf{x}) = [0.2, -0.1, 0.6]$, and stochastic noise $\bm{\xi} = [\xi_1, \xi_2, \xi_3]$, where $\xi_1$, $\xi_2$, $\xi_3$ are iid Gaussian random variables with variance scaling with $0.2, 1, 2$, respectively. Scaling is used to vary the SNR. We generate $100$ adversaries by solving $\argmax_{\| \bm{\epsilon} \| \leq \rho} \langle \nabla f(\mathbf{x}) + \bm{\xi}, \bm{\epsilon} \rangle$ for each SNR value. As shown in Fig. \ref{fig.noise}, the adversaries are unlikely to capture the sharpness information when the SNR is small, because they spread indistinguishably all over the sphere.

\section{More on $m$-sharpness}\label{apdx.sec.m-sharpness}

\textbf{$m$-sharpness can be ill-posed.}
Our reason for not studying $m$-sharpness directly is that its formulation \citep[eq. (3)]{maksym2022} may be ill-posed mathematically due to the lack of a clear definition on how the dataset ${\cal S}$ is partitioned. Consider the following example, where the same notation as \citep{maksym2022} is adopted for convenience. Suppose that the loss function is $l_i(w) = a_i w^2 + b_i w $, where $\{(a_i, b_i)\}_i$ are data points, and $w$ is the parameter to be optimized. Let the dataset have $4$ samples, $(a_1=0, b_1=1)$; $(a_2=0, b_2=-1)$; $(a_3=-1, b_3=0)$; and, $(a_4=1, b_4=0)$. Consider 2-sharpness.
\begin{itemize}
	\item If the data partition is \{1,2\} and \{3,4\}, the objective of 2-sharpness, namely equation (3) in \citep{maksym2022}, becomes $\min_w \sum_{i=1}^2 \max_{||\delta|| < \rho} 0$.
	\item If the data partition is \{1,3\} and \{2,4\}, the objective is $\min_w \sum_{i=1}^2 \max_{||\delta|| < \rho} f_i(w,\delta)$, where $f_1(w,\delta) = -(w+\delta)^2 + (w+\delta)$ is the loss on partition \{1,3\}; and $f_2(w,\delta) = (w + \delta)^2 - (w + \delta)$ is the loss on partition \{2,4\}.
\end{itemize}
The objective functions are different when the data partition varies. This makes the problem ill-posed -- data partition leads to entirely different loss curvature. In practice, the data partition even varies across epochs due to the random shuffle in data loading. 

\section{Details on numerical tests}\label{apdx.sec.numerical}

\subsection{ImageNet}
\textbf{ResNet50.}
Due to limitations of computational resources, we report the averaged results over $2$ independent runs. For this dataset, we randomly resize and crop all images to a resolution of $224\times 224$, and apply random horizontal flip, as well as normalization during training. The batch size is $128$ with a cosine learning rate scheduling and with initial step size $0.05$. The momentum and weight decay of the base optimizer, here SGD, are set to $0.9$ and $10^{-4}$, respectively. We tune $\rho$ from $\{0.05, 0.075, 0.1, 0.2\}$, and choose $\rho=0.075$ for SAM. VASSO uses $\theta=0.99$. ASAM, F-SAM and VASSO adopt the same $\rho=0.075$.

\textbf{ViT-S/32.} We follow the implementation of \citep{du2022saf}, where we train the model for $300$ epochs with a batch size of $4096$. The baseline optimizer is chosen as AdamW with weight decay $0.3$. SAM relies on $\rho=0.05$. For the implementation of GSAM and V+G, we adopt the same implementation from \citep{zhuang2022}.

\subsection{Domain generalization}\label{apdx.sec.domainbed}
Table \ref{tab.domain-summary} includes details of the DomainBed benchmark \citep{gulrajani2022}; while 
Tables \ref{tab.PACS}, \ref{tab.VLCS}, \ref{tab.OfficeHome}, and \ref{tab.TerraInc} provide leave-one-out cross-validation results per dataset. 


\begin{table}[t]
    \centering
    \caption{Summary of DomainBed benchmark}
    \tabcolsep=0.5cm
    \vspace{-0.2cm}
    \begin{tabular}{l|cc}
    \toprule
      dataset  & \# images & \# classes  \\
   \midrule
       PACS  & 9,991 & 7  \\
       VLCS  & 10,729 & 5  \\
       OfficeHome & 15,588 & 65 \\
       TerraInc & 24,788 & 10 \\
    \bottomrule
    \end{tabular}
    \label{tab.domain-summary}
    \vspace{-.1cm}
\end{table}

\begin{table}[ht]
	\centering
	\caption{Cross-validation accuracies (\%) on PACS dataset.}
    \vspace{-.2cm}
	\begin{tabular}{l|cccc}
	\toprule
	Test domain  & Adam             & SAM              & GSAM               & VASSO    \\
	\midrule 
	\underline{P}hoto  & 97.2$_{\pm0.3}$  & 97.0$_{\pm0.4}$  & 97.5$_{\pm0.0}$           &  95.4$_{\pm0.2}$  \\
	\underline{A}rt painting  &  84.7$_{\pm0.4}$ & 85.6$_{\pm2.1}$  & 86.9$_{\pm0.1}$  &  84.2$_{\pm0.6}$ \\
	\underline{C}artoon  &  80.8$_{\pm0.6}$ & 80.9$_{\pm1.2}$  & 80.4$_{\pm0.2}$    &  81.8$_{\pm0.5}$ \\
	\underline{S}ketch  &  79.3$_{\pm1.0}$ & 79.6$_{\pm1.6}$  & 78.7$_{\pm0.8}$    &  82.5$_{\pm0.4}$ \\
	\midrule 
	Average   & 85.5 & 85.8  & 85.9  & \textbf{86.0}  \\
	\bottomrule
	\end{tabular}
	\vspace{-0.1cm}
	\label{tab.PACS}
\end{table}

\begin{table}[ht]
	\centering
	\caption{Cross-validation accuracies (\%) on VLCS dataset.}
    \vspace{-.2cm}
	\begin{tabular}{l|cccc}
	\toprule
	Test domain  & Adam             & SAM              & GSAM               & VASSO    \\
	\midrule 
	\underline{V}OC2007  &  75.2$_{\pm1.6}$ & 79.8$_{\pm0.1}$  & 78.5$_{\pm0.8}$           &  78.9$_{\pm0.7}$  \\
	\underline{L}abelMe  &  64.7$_{\pm1.2}$ & 65.0$_{\pm1.0}$  & 64.9$_{\pm0.2}$  &  65.9$_{\pm0.5}$ \\
	\underline{C}altech101  &  98.0$_{\pm0.3}$ & 99.1$_{\pm0.2}$  & 98.7$_{\pm0.3}$    &  99.0$_{\pm0.2}$ \\
	\underline{S}UN09  & 71.4$_{\pm1.2}$  & 73.7$_{\pm1.0}$  & 74.3$_{\pm0.0}$    &  74.4$_{\pm0.9}$ \\
	\midrule 
	Average   & 77.3 & 79.4  & 79.1  & \textbf{79.6}  \\
	\bottomrule
	\end{tabular}
	\vspace{-0.1cm}
	\label{tab.VLCS}
\end{table}

\begin{table}[ht]
	\centering
	\caption{Cross-validation accuracies (\%) on OfficeHome dataset.}
    \vspace{-.2cm}
	\begin{tabular}{l|cccc}
	\toprule
	Test domain  & Adam             & SAM              & GSAM               & VASSO    \\
	\midrule 
	Art  & 61.3$_{\pm0.7}$  & 64.5$_{\pm0.3}$  & 64.9$_{\pm0.1}$   &  64.8$_{\pm0.5}$  \\
	Clipart  &  52.4$_{\pm0.3}$ & 56.5$_{\pm0.2}$  & 55.2$_{\pm0.2}$  &  57.1$_{\pm0.4}$ \\
	Product  &  75.8$_{\pm0.1}$ & 77.4$_{\pm0.1}$  & 77.8$_{\pm0.0}$    &  78.2$_{\pm0.2}$ \\
	Real world  &  76.6$_{\pm0.3}$ & 79.8$_{\pm0.4}$  & 79.2$_{\pm0.2}$    &  79.3$_{\pm0.1}$ \\
	\midrule 
	Average   & 66.5 & 69.6  & 69.3  & \textbf{69.8}  \\
	\bottomrule
	\end{tabular}
	\vspace{-0.1cm}
	\label{tab.OfficeHome}
\end{table}

\begin{table}[ht]
	\centering
	\caption{Cross-validation accuracies (\%) on TerraIncognita dataset.}
    \vspace{-.2cm}
	\begin{tabular}{l|cccc}
	\toprule
	Test domain  & Adam             & SAM              & GSAM               & VASSO    \\
	\midrule 
	Location 38  & 42.1$_{\pm1.4}$  & 38.4$_{\pm2.4}$  & 39.3$_{\pm0.2}$  &  38.5$_{\pm0.7}$ \\
	Location 43  & 56.9$_{\pm1.8}$  & 54.0$_{\pm1.0}$  & 59.6$_{\pm0.0}$    &  60.7$_{\pm0.2}$ \\
	Location 46  & 35.7$_{\pm3.9}$  & 34.5$_{\pm0.8}$  & 38.2$_{\pm0.8}$    &  37.5$_{\pm0.8}$ \\
 	Location 100  & 49.8$_{\pm4.4}$  & 46.3$_{\pm1.0}$  & 50.8$_{\pm0.1}$   &  51.4$_{\pm0.8}$  \\
	\midrule 
	Average   & 46.1 & 43.3  & \textbf{47.0}  & \textbf{47.0}  \\
	\bottomrule
	\end{tabular}
	\vspace{-0.1cm}
	\label{tab.TerraInc}
\end{table}

\section{Missing proofs}\label{apdx.sec.proof}

Alg. \ref{alg.sam} can be written as
\begin{subequations}\label{eq.alg_rewrite}
\begin{align}
	\mathbf{x}_{t+\frac{1}{2}} &= \mathbf{x}_t + \bm{\epsilon}_t \\
	 \mathbf{x}_{t+1} &= \mathbf{x}_t -  \eta_t \mathbf{g}_t(\mathbf{x}_{t+\frac{1}{2}})
\end{align}
\end{subequations}
where $\| \bm{\epsilon}_t\| = \rho$. In SAM, we have $\bm{\epsilon}_t = \rho \frac{ \mathbf{g}_t(\mathbf{x}_t)}{\| \mathbf{g}_t(\mathbf{x}_t) \|}$, and in VASSO we have $\bm{\epsilon}_t = \rho \frac{ \mathbf{d}_t}{\| \mathbf{d}_t \|}	$.

\subsection{Useful lemmas}
This subsection presents useful lemmas for our main results.

\begin{lemma}\label{apdx.lemma1}
	Alg. \ref{alg.sam} (or equivalently iteration \eqref{eq.alg_rewrite}) ensures that
	\begin{align*}
		& \eta_t \mathbb{E} \big[ \langle \nabla f(\mathbf{x}_t), \nabla f(\mathbf{x}_t) - \mathbf{g}_t(\mathbf{x}_{t+\frac{1}{2}})  \rangle \big] \\
		&~~~~~~~~~~~~~~~~~~~~~~~~~~~~~~~~~~ \leq \frac{L\eta_t^2 }{2} \mathbb{E} \big[ \| \nabla f(\mathbf{x}_t)\|^2 \big]  + \frac{L\rho^2}{2}.
	\end{align*}
\end{lemma}
\begin{proof}
	To start, we have that
	\begin{align*}
		& ~~~~~ \big\langle \nabla f(\mathbf{x}_t), \nabla f(\mathbf{x}_t) - \mathbf{g}_t(\mathbf{x}_{t+\frac{1}{2}})  \big\rangle \\
		& = \langle \nabla f(\mathbf{x}_t), \nabla f(\mathbf{x}_t) - \mathbf{g}_t(\mathbf{x}_t) + \mathbf{g}_t(\mathbf{x}_t) - \mathbf{g}_t(\mathbf{x}_{t+\frac{1}{2}})  \rangle.
	\end{align*}
	Taking expectation conditioned on $\mathbf{x}_t$, we arrive at
	\begin{align}
		& ~~~~~ \mathbb{E} \big[ \big\langle \nabla f(\mathbf{x}_t), \nabla f(\mathbf{x}_t) - \mathbf{g}_t(\mathbf{x}_{t+\frac{1}{2}})  \big\rangle | \mathbf{x}_t \big] \nonumber \\
		& = \mathbb{E} \big[ \langle \nabla f(\mathbf{x}_t), \nabla f(\mathbf{x}_t) - \mathbf{g}_t(\mathbf{x}_t)   \rangle | \mathbf{x}_t \big] \nonumber \\
		& ~~~~~~~~~~~~~~~~ + \mathbb{E} \big[ \langle \nabla f(\mathbf{x}_t), \mathbf{g}_t(\mathbf{x}_t) - \mathbf{g}_t(\mathbf{x}_{t+\frac{1}{2}})  \rangle | \mathbf{x}_t \big] \nonumber \\
		& = \mathbb{E} \big[ \langle \nabla f(\mathbf{x}_t), \mathbf{g}_t(\mathbf{x}_t) - \mathbf{g}_t(\mathbf{x}_{t+\frac{1}{2}})  \rangle | \mathbf{x}_t \big]  \nonumber \\
		& \leq \mathbb{E} \big[ \| \nabla f(\mathbf{x}_t)\|  \cdot \| \mathbf{g}_t(\mathbf{x}_t) - \mathbf{g}_t(\mathbf{x}_{t+\frac{1}{2}}) \| | \mathbf{x}_t \big] \nonumber \\
		& \stackrel{(a)}{\leq} L  \mathbb{E} \big[ \| \nabla f(\mathbf{x}_t)\| \cdot \| \mathbf{x}_t - \mathbf{x}_{t+\frac{1}{2}} \| | \mathbf{x}_t \big] \nonumber \\
		& \stackrel{(b)}{=}  L  \rho  \| \nabla f(\mathbf{x}_t)\|  \nonumber
	\end{align}
	where (a) follows from Assumption \ref{as.2}; and (b) holds because $\mathbf{x}_t - \mathbf{x}_{t+\frac{1}{2}} = -\bm{\epsilon}_t$, and its norm equals $\rho$. This inequality ensures that
	\begin{align*}
		& ~~~~~ \eta_t 	\mathbb{E} \big[ \big\langle \nabla f(\mathbf{x}_t), \nabla f(\mathbf{x}_t) - \mathbf{g}_t(\mathbf{x}_{t+\frac{1}{2}})  \big\rangle | \mathbf{x}_t \big] \\
		& \leq  L  \rho \eta_t \| \nabla f(\mathbf{x}_t)\| \leq \frac{L\eta_t^2 \| \nabla f(\mathbf{x}_t)\|^2 }{2} + \frac{L\rho^2}{2}
	\end{align*}
	where the last inequality holds because $\rho \eta_t \| \nabla f(\mathbf{x}_t)\| \leq \frac{1}{2} \eta_t^2 \| \nabla f(\mathbf{x}_t)\|^2 + \frac{1}{2} \rho^2 $. Taking expectation w.r.t. $\mathbf{x}_t$ completes the proof.
\end{proof}

\begin{lemma}\label{apdx.lemma2}
	Alg. \ref{alg.sam} (or equivalently iteration \eqref{eq.alg_rewrite}) ensures that
	\begin{align*}
		\mathbb{E} \big[ \| \mathbf{g}_t(\mathbf{x}_{t+\frac{1}{2}}) \|^2 \big] \leq 2 L^2 \rho^2 + 2 \mathbb{E} \big[ \|  \nabla f(\mathbf{x}_t) \|^2 \big] + 2 \sigma^2.
	\end{align*}
\end{lemma}
\begin{proof}
	The proof starts with bounding $\| \mathbf{g}_t(\mathbf{x}_{t+\frac{1}{2}})\|$ as
	\begin{align*}
		\| \mathbf{g}_t(\mathbf{x}_{t+\frac{1}{2}}) \|^2 & = \| \mathbf{g}_t(\mathbf{x}_{t+\frac{1}{2}})  - \mathbf{g}_t(\mathbf{x}_t) + \mathbf{g}_t(\mathbf{x}_t) \|^2 \\
		& \leq 2 \| \mathbf{g}_t(\mathbf{x}_{t+\frac{1}{2}}) - \mathbf{g}_t(\mathbf{x}_t) \|^2 + 2 \| \mathbf{g}_t (\mathbf{x}_t) \|^2 \\
		& \stackrel{(a)}{\leq}  2 L^2 \|\mathbf{x}_t - \mathbf{x}_{t+\frac{1}{2}} \|^2 + 2 \| \mathbf{g}_t(\mathbf{x}_t) \|^2  \\
		& \stackrel{(b)}{=} 2 L^2 \rho^2 + 2 \| \mathbf{g}_t(\mathbf{x}_t) - \nabla f(\mathbf{x}_t) + \nabla f(\mathbf{x}_t) \|^2  
	\end{align*}
	where (a) follows from Assumption \ref{as.2}; and (b) is because $\mathbf{x}_t - \mathbf{x}_{t+\frac{1}{2}} = -\bm{\epsilon}_t$, and its norm equals $\rho$. 
	
Taking expectation conditioned on $\mathbf{x}_t$, we have
	\begin{align*}
		& ~~~~~ \mathbb{E}\big[ \| \mathbf{g}_t( \mathbf{x}_{t+\frac{1}{2}}) \|^2 | \mathbf{x}_t \big]	 \\
		& \leq 2 L^2 \rho^2 + 2 \mathbb{E}\big[ \| \mathbf{g}_t(\mathbf{x}_t) - \nabla f(\mathbf{x}_t) + \nabla f(\mathbf{x}_t) \|^2 | \mathbf{x}_t \big] \nonumber \\
		& \leq 2 L^2 \rho^2 + 2 \|  \nabla f(\mathbf{x}_t) \|^2 + 2 \sigma^2 \nonumber
	\end{align*}
	where the last inequality is because of Assumption \ref{as.3}. Taking expectation w.r.t. the randomness in $\mathbf{x}_t$ completes the proof.
\end{proof}

\begin{lemma}\label{apdx.lemma3}
With $A_{t+1} = \alpha A_t + \beta$ for some $\alpha \in (0, 1)$, we have
	\begin{align*}
		A_{t+1} \leq \alpha^{t+1} A_0 + \frac{\beta}{1-\alpha}.
	\end{align*}
\end{lemma}
\begin{proof}
	The proof can be completed by simply unrolling $A_{t+1}$ and using the fact that $1 + \alpha + \alpha^2 + \ldots + \alpha^t \leq \frac{1}{1 - \alpha }$.	
\end{proof}

\subsection{Proof of Theorem \ref{thm.sam}}
\begin{proof}
	Using Assumption \ref{as.2}, we have that
	\begin{align*}
		&~~~~~ f(\mathbf{x}_{t+1}) - f (\mathbf{x}_t)  \\
		& \leq 	\langle \nabla f(\mathbf{x}_t), \mathbf{x}_{t+1} - \mathbf{x}_t \rangle + \frac{L}{2} \| \mathbf{x}_{t+1} - \mathbf{x}_t \|^2 \\
		& = - \eta_t \langle \nabla f(\mathbf{x}_t), \mathbf{g}_t(\mathbf{x}_{t+\frac{1}{2}})  \rangle + \frac{L\eta_t^2}{2} \| \mathbf{g}_t(\mathbf{x}_{t+\frac{1}{2}}) \|^2 \nonumber \\
		& = - \eta_t \langle \nabla f(\mathbf{x}_t), \mathbf{g}_t(\mathbf{x}_{t+\frac{1}{2}}) - \nabla f(\mathbf{x}_t) + \nabla f(\mathbf{x}_t)  \rangle  \\
		& ~~~~~~~~~~~~~~~~~~~~~~~~~~~~~~~~~~~~~~~~~~~~~~~~~~~~~~~ + \frac{L\eta_t^2}{2} \| \mathbf{g}_t(\mathbf{x}_{t+\frac{1}{2}})\|^2 \nonumber \\
		& = - \eta_t \| \nabla f(\mathbf{x}_t) \|^2 - \eta_t \langle \nabla f(\mathbf{x}_t), \mathbf{g}_t(\mathbf{x}_{t+\frac{1}{2}}) - \nabla f(\mathbf{x}_t)  \rangle \\
		& ~~~~~~~~~~~~~~~~~~~~~~~~~~~~~~~~~~~~~~~~~~~~~~~~~~~~~~~ + \frac{L\eta_t^2}{2} \| \mathbf{g}_t(\mathbf{x}_{t+\frac{1}{2}}) \|^2 .
	\end{align*}
	
	Taking expectation, and plugging Lemmas \ref{apdx.lemma1} and \ref{apdx.lemma2}, yields
	\begin{align*}
		 \mathbb{E}\big[  f(\mathbf{x}_{t+1}) - f (\mathbf{x}_t)  \big] & \leq -  \bigg( \eta_t - \frac{3L\eta_t^2}{2} \bigg) \mathbb{E}\big[ \| \nabla f(\mathbf{x}_t )\|^2 \big] \\
		 & ~~~~~~~ + \frac{L\rho^2}{2}  + L^3 \eta_t^2 \rho^2  + L \eta_t^2 \sigma^2.
	\end{align*}

	As the parameter selection ensures that $\eta_t \equiv \eta = \frac{\eta_0}{ \sqrt{T}} \leq \frac{2}{3L}$, we can divide both sides by $\eta$, and rearrange terms to arrive at
	\begin{align*}
		& ~~~~~ \bigg( 1 - \frac{3L\eta}{2} \bigg) \mathbb{E}\big[ \| \nabla f(\mathbf{x}_t )\|^2 \big] \\
		 & \leq \frac{\mathbb{E}\big[  f (\mathbf{x}_t) - f(\mathbf{x}_{t+1})  \big]}{\eta}  + \frac{L\rho^2}{2 \eta }  + L^3 \eta \rho^2  + L \eta \sigma^2.
	\end{align*}
	Summing over $t$, we have
	\begin{align*}
		& ~~~~~ \bigg( 1 - \frac{3L\eta}{2} \bigg) \frac{1}{T}\sum_{t=0}^{T-1}\mathbb{E}\big[ \| \nabla f(\mathbf{x}_t )\|^2 \big] \\
		& \leq \frac{\mathbb{E}\big[  f (\mathbf{x}_0) - f(\mathbf{x}_T)  \big]}{\eta T}  + \frac{L\rho^2}{2 \eta }  + L^3 \eta \rho^2  + L \eta \sigma^2 \\
		& \stackrel{(a)}{\leq} \frac{  f (\mathbf{x}_0) - f^*  }{\eta T}  + \frac{L\rho^2}{2 \eta }  + L^3 \eta \rho^2  + L \eta \sigma^2 \nonumber \\
		& = \frac{  f (\mathbf{x}_0) - f^*  }{\eta_0 \sqrt{T}}  +  \frac{L\rho_0^2}{2 \eta_0 \sqrt{T}}  + \frac{L^3 \eta_0 \rho_0^2}{T^{3/2}}  + \frac{L \eta_0 \sigma^2}{\sqrt{T}} \nonumber 
	\end{align*}
	where (a) uses Assumption \ref{as.1}, and the last equality is obtained by plugging in the values of $\rho$ and $\eta$. This completes the proof of the first part. 
	
	For the second part of this theorem, we have that
	\begin{align*}
		& ~~~~~ \mathbb{E}\big[ \| \nabla f(\mathbf{x}_t + \bm{\epsilon}_t )\|^2 \big] \\
		& = \mathbb{E}\big[ \| \nabla f(\mathbf{x}_t + \bm{\epsilon}_t ) +\nabla f(\mathbf{x}_t) - \nabla f(\mathbf{x}_t) \|^2 \big] \\
		& \leq  2 \mathbb{E}\big[ \| \nabla f(\mathbf{x}_t \|^2 \big] + 2 \mathbb{E}\big[ \| \nabla f(\mathbf{x}_t + \bm{\epsilon}_t ) - \nabla f(\mathbf{x}_t) \|^2 \big] \\
		& \leq  2 \mathbb{E}\big[ \| \nabla f(\mathbf{x}_t \|^2 \big] + 2 L^2 \rho^2 \\
		& = 2 \mathbb{E}\big[ \| \nabla f(\mathbf{x}_t \|^2 \big] +  \frac{2 L^2 \rho_0^2}{T}.
	\end{align*}
	Averaging over $t$ completes the proof.
	
	\noindent\textbf{Extension}. Theorem \ref{thm.sam} also holds when $\bm{\epsilon}_t \in \mathbb{S}_\rho(\bm{0})$ is replaced by $\bm{\epsilon}_t \in \mathbb{B}_\rho(\bm{0})$. Since the derivation is similar, we will skip it here.
	
\end{proof}

\subsection{Proof of Theorem \ref{thm.vso}}
\begin{proof}
	To bound the MSE, we first have that
	\begin{align}\label{apdx.eq.mse}
		& ~~~~~ \| \mathbf{d}_t - \nabla f(\mathbf{x}_t) \|	^2 \\
		& = \| (1 - \theta)	\mathbf{d}_{t-1} + \theta \mathbf{g}_t(\mathbf{x}_t) - (1 - \theta)  \nabla f(\mathbf{x}_t) -  \theta \nabla f(\mathbf{x}_t) \|^2   \nonumber \\
		& = (1-\theta)^2 \| \mathbf{d}_{t-1} - \nabla f(\mathbf{x}_t) \|^2 + \theta^2 \|\mathbf{g}_t(\mathbf{x}_t)  - \nabla f(\mathbf{x}_t) \|^2  \nonumber \\
		& ~~~~~~~~~~~~~~~~   + 2 \theta(1-\theta ) \langle \mathbf{d}_{t-1} - \nabla f(\mathbf{x}_t),  \mathbf{g}_t(\mathbf{x}_t)  - \nabla f(\mathbf{x}_t) \rangle. \nonumber 
	\end{align}
We will cope with three terms in the right hand side (rhs) of \eqref{apdx.eq.mse} separately. 
	
	The second term can be bounded directly using Assumption \ref{as.2}
	\begin{align}\label{apdx.eq.mse2}
		\mathbb{E}\big[  \|\mathbf{g}_t(\mathbf{x}_t)  - \nabla f(\mathbf{x}_t) \|^2  | \mathbf{x}_t \big] \leq \sigma^2.
	\end{align}
	
	For the third term, we have 
	\begin{align}\label{apdx.eq.mse3}
		\mathbb{E}\big[ \langle \mathbf{d}_{t-1} - \nabla f(\mathbf{x}_t),  \mathbf{g}_t(\mathbf{x}_t)  - \nabla f(\mathbf{x}_t) \rangle | \mathbf{x}_t \big] = 0.
	\end{align}

	The first term is bounded through
	\begin{align*}
		& ~~~~~ \| \mathbf{d}_{t-1} - \nabla f(\mathbf{x}_t)	 \|^2 \\
		& = \| \mathbf{d}_{t-1} - \nabla f(\mathbf{x}_{t-1}) + \nabla f(\mathbf{x}_{t-1}) - \nabla f(\mathbf{x}_t) \|^2 \\
		& \stackrel{(a)}{\leq} (1 \! + \! \lambda)\| \mathbf{d}_{t\!-\!1} \!- \!\nabla f(\mathbf{x}_{t\!-\!1}) \|^2 + \big(1 \! + \! \frac{1}{\lambda} \big) \| \nabla f(\mathbf{x}_{t\!-\!1}) \! - \! \nabla f(\mathbf{x}_t) \|^2 \nonumber \\
		& \leq (1 + \lambda)\| \mathbf{d}_{t-1} - \nabla f(\mathbf{x}_{t-1}) \|^2 + \big(1+ \frac{1}{\lambda}\big) L^2 \| \mathbf{x}_{t-1} -\mathbf{x}_t \|^2  \nonumber \\
		& = (1 \! + \! \lambda)\| \mathbf{d}_{t-1} - \nabla f(\mathbf{x}_{t-1}) \|^2 + \big(1 \! + \!\frac{1}{\lambda}\big) \eta^2 L^2  \| \mathbf{g}_{t-1}(\mathbf{x}_{t-\frac{1}{2}}) \|^2 
	\end{align*}
	where (a) follows from Young's inequality. Taking expectation and applying Lemma \ref{apdx.lemma2}, we arrive at
	\begin{align}\label{apdx.eq.mse1}
		& ~~~~ \mathbb{E} \big[ \| \mathbf{d}_{t-1} - \nabla f(\mathbf{x}_t)	 \|^2 \big] \\
		& \leq (1 + \lambda) \mathbb{E} \big[  \| \mathbf{d}_{t-1} - \nabla f(\mathbf{x}_{t-1}) \|^2 \big] \nonumber \\
		& ~~~~~ + \big(1+ \frac{1}{\lambda}\big) \eta^2L^2 \bigg( 2 L^2 \rho^2 + 2  \mathbb{E}\big[ \| \nabla f(\mathbf{x}_{t-1}) \|^2 \big] + 2\sigma^2 \bigg) \nonumber \\
		& \leq (1 + \lambda) \mathbb{E} \big[  \| \mathbf{d}_{t-1} - \nabla f(\mathbf{x}_{t-1}) \|^2 \big] + \big(1+ \frac{1}{\lambda}\big)  \cdot {\cal O}\bigg(\frac{\sigma^2}{\sqrt{T}}\bigg). \nonumber
	\end{align}

	The last inequality uses the value of $\eta= \frac{\eta_0}{ \sqrt{T}}$ and $\rho= \frac{\rho_0}{ \sqrt{T}}$. In particular, we have $\eta^2 \rho^2 L^4 = {\cal O}(1/T^2)$ and $\eta^2 L^2\sigma^2 = {\cal O}(\sigma^2/T)$, and 
	\begin{align*}
		&~~~~~ \eta^2 L^2 \mathbb{E}\big[ \| \nabla f(\mathbf{x}_t) \|^2 \big]  = \frac{\eta_0^2 L^2}{T} \mathbb{E}\big[ \| \nabla f(\mathbf{x}_t) \|^2 \big] \\
		& \leq \eta_0^2 L^2 \frac{1}{T} \sum_{t=0}^{T-1} \mathbb{E}\big[ \| \nabla f(\mathbf{x}_t) \|^2 \big]  = {\cal O}\bigg( \frac{\sigma^2}{\sqrt{T}} \bigg)
	\end{align*}
	where for the last equality we have also used Theorem \ref{thm.sam}.

Combining \eqref{apdx.eq.mse} with \eqref{apdx.eq.mse1}, \eqref{apdx.eq.mse2} with \eqref{apdx.eq.mse3}, and choosing $\lambda = \frac{\theta}{1 - \theta}$, we have 
	\begin{align*}
		& ~~~~~ \mathbb{E} \big[ \| \mathbf{d}_t - \nabla f(\mathbf{x}_t) \|	^2 \big] \\
		& \leq (1 \!- \! \theta) \mathbb{E} \big[  \| \mathbf{d}_{t \!-\!1} - \nabla f(\mathbf{x}_{t\!-\!1}) \|^2 \big] + \frac{(1-\theta)^2}{\theta}{\cal O}\bigg(\frac{\sigma^2}{\sqrt{T}}\bigg) \! + \! \theta^2 \sigma^2 \nonumber \\
		& \leq \theta \sigma^2 + {\cal O}\bigg(\frac{(1 \! -\! \theta)^2 \sigma^2}{ \theta^2 \sqrt{T}}\bigg)
	\end{align*}
	where the last inequality is the result of Lemma \ref{apdx.lemma3}.
\end{proof}

\subsection{Proof of Theorem \ref{thm.salad}}

\begin{proof}
We adopt a unified notation for simplicity. Let $\mathbf{v}_t:= \mathbf{d}_t$ for VASSO, and $\mathbf{v}_t:= \mathbf{g}_t(\mathbf{x}_t)$ for SAM. Then for both VASSO and SAM, we can write
	\begin{align}\label{apdx.thm3.eq1}
		 & ~~~~~ f(\mathbf{x}_t) + \langle \mathbf{v}_t, \bm{\epsilon}_t \rangle = f(\mathbf{x}_t) + \rho \| \mathbf{v}_t \| \\
		 & = f(\mathbf{x}_t) + \rho \| \mathbf{v}_t - \nabla f(\mathbf{x}_t) + \nabla f(\mathbf{x}_t) \|. \nonumber
	\end{align}
	
	For convenience, let $\bm{\epsilon}_t^* = \rho \nabla f(\mathbf{x}_t) / \| \nabla f(\mathbf{x}_t)\|$. From \eqref{apdx.thm3.eq1}, we have that 
	\begin{align}\label{apdx.thm3.eq2}
		 & ~~~~~ f(\mathbf{x}_t) + \langle \mathbf{v}_t, \bm{\epsilon}_t \rangle \\
		 & = f(\mathbf{x}_t) + \rho \| \mathbf{v}_t - \nabla f(\mathbf{x}_t) + \nabla f(\mathbf{x}_t) \| \nonumber \\
		 & \leq f(\mathbf{x}_t)  + \rho\| \nabla f(\mathbf{x}_t) \| + \rho \| \mathbf{v}_t - \nabla f(\mathbf{x}_t) \| \nonumber \\
		 & = f(\mathbf{x}_t) +  \langle \nabla f(\mathbf{x}_t), \bm{\epsilon}_t^* \rangle + \rho \| \mathbf{v}_t - \nabla f(\mathbf{x}_t) \|. \nonumber
	\end{align}
	
	Applying the triangle inequality $\big| \| \mathbf{a} \| - \| \mathbf{b} \|  \big| \leq \| \mathbf{a} - \mathbf{b} \|$, we arrive at
	\begin{align}\label{apdx.thm3.eq3}
		 &~~~~~ f(\mathbf{x}_t) + \langle \mathbf{v}_t, \bm{\epsilon}_t \rangle \\
		 & = f(\mathbf{x}_t) + \rho \|  \nabla f(\mathbf{x}_t) - (\nabla f(\mathbf{x}_t) - \mathbf{v}_t ) \| \nonumber \\
		 & \geq f(\mathbf{x}_t)  + \rho\| \nabla f(\mathbf{x}_t) \| - \rho \| \mathbf{v}_t - \nabla f(\mathbf{x}_t) \| \nonumber \\
		 & = f(\mathbf{x}_t) +  \langle \nabla f(\mathbf{x}_t), \bm{\epsilon}_t^* \rangle - \rho \| \mathbf{v}_t - \nabla f(\mathbf{x}_t) \|. \nonumber
	\end{align}

	Combining \eqref{apdx.thm3.eq2} with \eqref{apdx.thm3.eq3}, we have
	\begin{align*}
		| {\cal L}_t(\mathbf{v}_t) - 	{\cal L}_t(\nabla f( \mathbf{x}_t)) | \leq \rho  \| \mathbf{v}_t - \nabla f(\mathbf{x}_t) \|
	\end{align*}
	which further implies that
	\begin{align*}
		\mathbb{E} \big[ | {\cal L}_t(\mathbf{v}_t) - 	{\cal L}_t(\nabla f( \mathbf{x}_t)) | \big] & \leq \rho \mathbb{E} \big[ \| \mathbf{v}_t - \nabla f(\mathbf{x}_t) \| \big] \\
		& \leq \rho \sqrt{\mathbb{E} \big[ \| \mathbf{v}_t - \nabla f(\mathbf{x}_t) \|^2 \big]}
	\end{align*}
where the last inequality follows from $(\mathbb{E}[a])^2 \leq \mathbb{E}[a^2]$. This theorem can be proved by applying Assumption \ref{as.3} for SAM, and Lemma \ref{thm.vso} for VASSO.
\end{proof}

\subsection{Proof for Theorem \ref{thm.r-vso}}

We start with a technical lemma needed for the main theorem. 

\begin{lemma}\label{apdx.lemma2-rvasso}
	Alg. \ref{alg.r-vasso} ensures that
	\begin{align*}
		\mathbb{E} \big[ \| \mathbf{g}_t(\mathbf{x}_t + \bm{\epsilon}_t) \|^2 \big] \leq 2p L^2 \rho^2 + (1+p) \|  \nabla f(\mathbf{x}_t) \|^2 +  (1+p) \sigma^2.
	\end{align*}
\end{lemma}
\begin{proof}
	First, consider the case where $B_t=1$. We have that
	\begin{align*}
		\| \mathbf{g}_t(\mathbf{x}_t + \bm{\epsilon}_t) \|^2 & = \| \mathbf{g}_t(\mathbf{x}_t + \bm{\epsilon}_t)  - \mathbf{g}_t(\mathbf{x}_t) + \mathbf{g}_t(\mathbf{x}_t) \|^2 \\
		& \leq 2 \| \mathbf{g}_t(\mathbf{x}_t + \bm{\epsilon}_t) - \mathbf{g}_t(\mathbf{x}_t) \|^2 + 2 \| \mathbf{g}_t (\mathbf{x}_t) \|^2 \\
		& \stackrel{(a)}{\leq}  2 L^2 \| \bm{\epsilon}_t \|^2 + 2 \| \mathbf{g}_t(\mathbf{x}_t) \|^2  \\
		& \stackrel{(b)}{=} 2 L^2 \rho^2 + 2 \| \mathbf{g}_t(\mathbf{x}_t) - \nabla f(\mathbf{x}_t) + \nabla f(\mathbf{x}_t) \|^2
	\end{align*}
	where (a) is due to Assumption \ref{as.2}; and (b) is because $\| \bm{\epsilon}_t\| =\rho$.

	When $B_t=0$, i.e., $ \bm{\epsilon}_t = \bm{0}$ , we have that
	\begin{align*}
		\| \mathbf{g}_t(\mathbf{x}_t + \bm{\epsilon}_t) \|^2 & = \| \mathbf{g}_t(\mathbf{x}_t) \|^2 = \| \mathbf{g}_t(\mathbf{x}_t) - \nabla f(\mathbf{x}_t) + \nabla f(\mathbf{x}_t) \|^2.
	\end{align*}

	Taking expectation conditioned on $\mathbf{x}_t$, we have
	\begin{align}
		& ~~~~~ \mathbb{E}\big[ \| \mathbf{g}_t( \mathbf{x}_t + \bm{\epsilon}_t ) \|^2 | \mathbf{x}_t \big] \nonumber \\
	    & \leq 2 p L^2 \rho^2 + ( 1+p) \mathbb{E}\big[ \| \mathbf{g}_t(\mathbf{x}_t) - \nabla f(\mathbf{x}_t) + \nabla f(\mathbf{x}_t) \|^2 | \mathbf{x}_t \big] \nonumber \\
		& \leq 2p L^2 \rho^2 + (1+p) \|  \nabla f(\mathbf{x}_t) \|^2 + (1+p) \sigma^2 \nonumber
	\end{align}
	where the last inequality is because of Assumption \ref{as.3}. Taking expectation w.r.t. the randomness in $\mathbf{x}_t$ completes the proof.
\end{proof}

Now we are ready to prove Theorem \ref{thm.r-vso}.

\begin{proof}
	Similar to \eqref{apdx.eq.mse}, we have that for eVASSO
	\begin{align}\label{apdx.eq.mse-rvasso}
		&  \| \mathbf{d}_t - \nabla f(\mathbf{x}_t) \|	^2 \\
		& = (1-\theta)^2 \| \mathbf{d}_{t-1} - \nabla f(\mathbf{x}_t) \|^2 + \theta^2 \|\mathbf{g}_t(\mathbf{x}_t)  - \nabla f(\mathbf{x}_t) \|^2 \nonumber  \\
		& ~~~~~~~~~~~  + 2 \theta(1-\theta ) \langle \mathbf{d}_{t-1} - \nabla f(\mathbf{x}_t),  \mathbf{g}_t(\mathbf{x}_t)  - \nabla f(\mathbf{x}_t) \rangle. \nonumber 
	\end{align}
Again, we deal with the three terms in the rhs \eqref{apdx.eq.mse-rvasso} separately. 
	
	The second term can be bounded directly using Assumption \ref{as.2}
	\begin{align}\label{apdx.eq.mse2-rvasso}
		\mathbb{E}\big[  \|\mathbf{g}_t(\mathbf{x}_t)  - \nabla f(\mathbf{x}_t) \|^2  | \mathbf{x}_t \big] \leq \sigma^2.
	\end{align}
	
	For the third term, we have 
	\begin{align}\label{apdx.eq.mse3-rvasso}
		\mathbb{E}\big[ \langle \mathbf{d}_{t-1} - \nabla f(\mathbf{x}_t),  \mathbf{g}_t(\mathbf{x}_t)  - \nabla f(\mathbf{x}_t) \rangle | \mathbf{x}_t \big] = 0.
	\end{align}

	The first term is bounded through
	\begin{align*}
		& ~~~~~ \| \mathbf{d}_{t-1} - \nabla f(\mathbf{x}_t)	 \|^2 \\
		& = \| \mathbf{d}_{t-1} - \nabla f(\mathbf{x}_{t-1}) + \nabla f(\mathbf{x}_{t-1}) - \nabla f(\mathbf{x}_t) \|^2 \\
		& \stackrel{(a)}{\leq} (1 \!+\! \lambda)\| \mathbf{d}_{t\!-\!1} \! - \! \nabla f(\mathbf{x}_{t\!-\!1}) \|^2 + \big(1 \! + \! \frac{1}{\lambda} \big) \| \nabla f(\mathbf{x}_{t\!-\!1}) \! - \! \nabla f(\mathbf{x}_t) \|^2 \nonumber \\
		& \leq (1 \!+\! \lambda)\| \mathbf{d}_{t-1} - \nabla f(\mathbf{x}_{t-1}) \|^2 + \big(1+ \frac{1}{\lambda}\big) L^2 \| \mathbf{x}_{t-1} -\mathbf{x}_t \|^2  \nonumber \\
		& = (1 \! + \! \lambda)\| \mathbf{d}_{t\!-\!1} \!-\! \nabla f(\mathbf{x}_{t\!-\!1}) \|^2 + \big(1 \! + \! \frac{1}{\lambda}\big) \eta^2 L^2  \| \mathbf{g}_{t\!-\!1}(\mathbf{x}_{t\!-\!1} \! + \! \bm{\epsilon}_{t\!-\!1}) \|^2 
	\end{align*}
where (a) follows from Young's inequality. Taking expectation and applying Lemma \ref{apdx.lemma2-rvasso}, we deduce that
	\begin{align}\label{apdx.eq.mse1-rvasso}
		& ~~~~ \mathbb{E} \big[ \| \mathbf{d}_{t-1} - \nabla f(\mathbf{x}_t)	 \|^2 \big] \\
		& \leq (1 + \lambda) \mathbb{E} \big[  \| \mathbf{d}_{t-1} - \nabla f(\mathbf{x}_{t-1}) \|^2 \big] \nonumber \\
		& ~~~~ + \big(1+ \frac{1}{\lambda}\big) \eta^2L^2 \bigg( 2p L^2 \rho^2 + (1+p) \mathbb{E}\big[ \| \nabla f(\mathbf{x}_{t-1}) \|^2 \big] \bigg) \nonumber \\
		& ~~~~ + \big(1+ \frac{1}{\lambda}\big) \eta^2L^2 (1+p)\sigma^2  \nonumber \\
		& \leq (1 + \lambda) \mathbb{E} \big[  \| \mathbf{d}_{t-1} - \nabla f(\mathbf{x}_{t-1}) \|^2 \big] + \big(1+ \frac{1}{\lambda}\big)  \cdot {\cal O}\bigg(\frac{\sigma^2}{\sqrt{T}}\bigg). \nonumber
	\end{align}

	The last inequality uses the values of $\eta= \frac{\eta_0}{ \sqrt{T}}$ and $\rho= \frac{\rho_0}{ \sqrt{T}}$. In particular, we have $\eta^2 \rho^2 L^4 = {\cal O}(1/T^2)$ and $\eta^2 L^2\sigma^2 = {\cal O}(\sigma^2/T)$, and 
	\begin{align*}
		& ~~~~~ \eta^2 L^2 \mathbb{E}\big[ \| \nabla f(\mathbf{x}_t) \|^2 \big]  = \frac{\eta_0^2 L^2}{T} \mathbb{E}\big[ \| \nabla f(\mathbf{x}_t) \|^2 \big] \\
		& \leq \eta_0^2 L^2 \frac{1}{T} \sum_{t=0}^{T-1} \mathbb{E}\big[ \| \nabla f(\mathbf{x}_t) \|^2 \big] = {\cal O}\bigg( \frac{\sigma^2}{\sqrt{T}} \bigg)
	\end{align*}
	where the last equation is the result of Theorem \ref{thm.sam}.

	Combining \eqref{apdx.eq.mse-rvasso} with \eqref{apdx.eq.mse1-rvasso}, \eqref{apdx.eq.mse2-rvasso} with \eqref{apdx.eq.mse3-rvasso}, and choosing $\lambda = \frac{\theta}{1 - \theta}$, we have 
	\begin{align*}
		& ~~~~~ \mathbb{E} \big[ \| \mathbf{d}_t - \nabla f(\mathbf{x}_t) \|	^2 \big] \\
		& \leq (1 \!-\! \theta) \mathbb{E} \big[  \| \mathbf{d}_{t\!-\!1} - \! \nabla f(\mathbf{x}_{t\!-\!1}) \|^2 \big] + \frac{(1\!-\!\theta)^2}{\theta}{\cal O}\bigg(\frac{\sigma^2}{\sqrt{T}}\bigg) + \theta^2 \sigma^2 \nonumber \\
		& \leq \theta \sigma^2 + {\cal O}\bigg(\frac{(1-\theta)^2 \sigma^2}{ \theta^2 \sqrt{T}}\bigg)
	\end{align*}
	where the last inequality is the result of Lemma \ref{apdx.lemma3}.
\end{proof}

\end{document}